%% file: final_bandit_adaptation.tex
\documentclass{article}

\usepackage[numbers]{natbib}
\usepackage[final, nonatbib]{neurips_2019}
\setcitestyle{numbers}

\usepackage[utf8]{inputenc} 
\usepackage[T1]{fontenc}    
\usepackage{hyperref}       
\usepackage{url}            
\usepackage{booktabs}       
\usepackage{amsfonts}       
\usepackage{nicefrac}       
\usepackage{microtype}      

\usepackage{indentfirst}
\usepackage{algorithm}
\usepackage{algorithmic}
\usepackage{comment}
\usepackage{amsmath, amssymb, amsthm}
\usepackage{dsfont}
\usepackage{bigints}
\usepackage{graphicx}

\usepackage[utf8]{inputenc}
\usepackage{subcaption}
\usepackage[font=footnotesize,labelfont=bf]{caption}
\usepackage{upgreek}

\usepackage{xfrac}

\usepackage{color}
\usepackage{xcolor}
\definecolor{Red}{rgb}{1.00, 0.00, 0.00}
\definecolor{Blue}{rgb}{0.00, 0.00, 1.00}

\newtheorem{lemma}{Lemma}
\newtheorem{theorem}{Theorem}
\newtheorem{proposition}{Proposition}
\newtheorem{corollary}{Corollary}

\newtheorem{definition}{Definition}
\newtheorem*{remark*}{Remark}
\newtheorem*{proposition*}{Proposition}
\newtheorem*{corollary*}{Corollary}
\newtheorem*{theorem*}{Theorem}

\newcommand{\E}{\mathbb E}

\renewcommand{\P}{\mathbb P}
\newcommand{\Q}{\mathbb Q}
\newcommand{\R}{\mathbb R}
\newcommand{\algref}{\ref}
\newcommand{\floor}[1]{\left\lfloor #1 \right\rfloor }
\newcommand{\ceil}[1]{\left\lceil #1 \right\rceil}
\newcommand{\what}{\widehat}
\newcommand{\ifrac}[2]{#1/#2}
\renewcommand{\epsilon}{\varepsilon}
\newcommand{\hati}{\what{\imath}}

\newcommand{\Disc}{\textbf{Disc}}

\newcommand{\dint}{\, \text{d}}

\newcommand{\loc}{\text{loc}}

\newcommand{\algname}[1]{{#1}}

\newcommand{\rmin}{r_{\text{min}}}
\newcommand{\rcv}{r_{\text{CV}}}

\newcommand{\abs}[1]{\left\lvert #1 \right\rvert}
\renewcommand{\leq}{\leqslant}
\renewcommand{\geq}{\geqslant}
\newcommand{\Rt}{\overline R_T}

\DeclareMathOperator{\argmax}{\arg \max}
\DeclareMathOperator{\KL}{KL}

\title{Polynomial Cost of Adaptation for $\mathcal X$-Armed Bandits}

\author{%
   Hédi Hadiji \\
  Laboratoire de Mathématiques d'Orsay\\
  Université Paris-Sud, Orsay, France \\
  \texttt{hedi.hadiji@math.u-psud.fr} \\
}

\begin{document}

\maketitle

\begin{abstract}
	In the context of stochastic continuum-armed bandits, we present an algorithm that adapts to the unknown smoothness of the objective function. We exhibit and compute a \emph{polynomial cost of adaptation} to the Hölder regularity for regret minimization. To do this, we first reconsider the recent lower bound of \citet{locatelli2018adaptivity}, and define and characterize admissible rate functions. Our new algorithm matches any of these minimal rate functions. We provide a finite-time analysis and a thorough discussion about asymptotic optimality.
\end{abstract}


\input{final_bandit_adaptation_body}

\section*{Acknowledgements}
We would like to thank Gilles Stoltz and Pascal Massart for their valuable comments and suggestions.


\bibliography{biblio_adaptation}
\medskip

\clearpage
\input{final_bandit_adaptation_supp}

\small


\end{document}

%% file: final_bandit_adaptation_body.tex
\section{Introduction}

Multi-armed bandits are a well-known sequential learning problem. When the number of available decisions is large, some assumptions on the environment have to be made. In a vast line of work (see the literature discussion in Section~\ref{subsec:relworks}), these assumptions show up as nonparametric regularity conditions on the mean-payoff function. If this function is Hölder continuous with constant $L$ and exponent $\alpha$, and if the values of $L$ and $\alpha$ are given to the player, then natural strategies can ensure that the regret is upper bounded by
\begin{equation}\label{eq:1}
	L^{1 / (2\alpha + 1)} T^{(\alpha + 1)/(2\alpha+ 1)}\, .
\end{equation}
Of course, assuming that the player knows $\alpha$ and $L$ is often not realistic. Thus the need for \emph{adaptive} methods, that are agnostic with respect to the true regularity of the mean-payoff function. Unfortunately, \citet{locatelli2018adaptivity} recently showed that full adaptation is impossible, and that no algorithm can enjoy the same minimax guarantees as when the regularity is given to the player. We persevere and address the question:

{\centering
\emph{What can the player achieve when the true regularity is completely unknown? }\par
}

\paragraph{A polynomial cost of adaptation}
In statistics, minimax adaptation for nonparametric function estimation is a deep and active research domain. In many contexts, sharp adaptation is possible; often, an additional logarithmic factor in the error has to be paid when the regularity is unknown: this is known as the \emph{cost of adaptation}. See e.g., \citet{lepskii1991Problem}, \citet{birge1995Estimation}, \citet{massart2007Concentration} for adaptive methods, and \citet{cai2012Minimax} for a detailed survey of the topic. Under some more exotic assumptions ---see e.g., Example 3 of \citet{cai2005Adaptive} --- adapting is significantly harder: there may be a \emph{polynomial cost of adaptation}. 

In this paper, we show that in the sequential setting of multi-armed bandits, the necessary exploration forces a similar phenomenon, and we exhibit this polynomial cost of adaptation. To do so, we revisit the lower bounds of \citet{locatelli2018adaptivity}, and design a new algorithm that matches these lower bounds.

As a representative example of our results, our algorithm can achieve, without the knowledge of $\alpha$ and $L$, an unimprovable (up to logarithmic factors) regret bound of order 
\begin{equation}\label{eq:2}
L^{1/(1+ \alpha)} T^{(\alpha + 2)/(2\alpha + 2)} \, .
\end{equation}

\subsection{Related work}\label{subsec:relworks}

\paragraph{Continuum-armed bandits} 
Continuum-armed bandits, with nonparametric regularity assumptions, were introduced by \citet{agrawal1995ContinuumArmed}. \citet{kleinberg2004Nearly} established the minimax rates in the Hölder setting and introduced the CAB1 algorithm. \citet{auer2007improved} studied the problem with additional regularity assumptions under which the minimax rates are improved. Via different roads, \citet{bubeck2011Xarmed} and \citet{kleinberg2013Bandits} explored further generalizations of these types of regularity, namely the zooming dimension and the near-optimality dimension. \citet{bull2015Adaptivetreed} exhibited an algorithm that essentially adapts to some cases when the near-optimality dimension is zero.

In all these articles, the mean-payoff function needs to satisfy simultaneously two sets of regularity conditions. The first type is a usual Hölder condition, which ensures that the function does not vary too much around (one of) its maxima. The second type is a ``margin condition'' that lower bounds the number of very suboptimal arms; in the literature these are defined in many technically different ways. Adapting to the margin conditions is often possible when the Hölder regularity is known. However, all these algorithms require some prior knowledge about the Hölder regularity.

In this paper, we focus on the problem of adapting to Hölder regularity. Accordingly, we call \emph{adaptive} the algorithms that assume no knowledge of the Hölder exponent nor of the Lipschitz constant.

\paragraph{Adaptation for cumulative regret}
\citet{bubeck2011Lipschitza} introduced the problem of adaptation, and adapted to the Lipschitz constant under extra requirements. An important step was made in \citet{locatelli2018adaptivity}, where it is shown that adaptation at the classical minimax rates is impossible. In the same paper, the authors exhibited some conditions under which full adaptation is achievable, e.g., with knowledge of the value of the maximum, or when the near-optimality dimension is zero.

\paragraph{Other settings} 
For simple regret, the objections against adaptation do not hold, as the objective does not penalize exploration. Adaptation up to polylog factors is done with various (meta-)algorithms. \citet{locatelli2018adaptivity} sketch out an aggregation approach inspired by Lepski's method, while \citet{valko2013Stochastic}, \citet{grill2015Blackbox}, \citet{shang2019General} describe cross-validation methods thanks to which they adapt to the near-optimality dimension with unknown smoothness. As it turns out, this last approach yields clean results with our smoothness assumptions; we write the details in Appendix~\ref{app:simpleregret}.

Recently, \citet{krishnamurthy2019Contextual} studied continuum-armed contextual bandits and use a sophisticated aggregation scheme to derive an algorithm that adapts to the Lipschitz constant when $L \geq 1$.

\subsection{Contributions and outline}

In this paper, we fully compute the cost of adaptation for bandits with Hölder regularity. In Section~\ref{sec:setup} we discuss the adaptive (and nonadaptive) lower bounds. We take an asymptotic stance in order to precisely define the objective of adaptation. Doing so, we uncover a family of noncomparable lower bounds for adaptive algorithms (Theorem~\ref{thm:admrates}), and define the corresponding notion of optimality: admissibility.

Section~\ref{sec:alg} contains our main contribution: an admissible adaptive algorithm. We first recall the \algname{CAB1} algorithm, which is nonadaptive minimax, and use it as a building block for our new algorithm (Subsection~\ref{subsec:cab1.1}). This algorithm works in a regime-based fashion. Between successive regimes of doubling lengths, we reset the algorithm and use a new discretization with fewer arms. In order to carry information between the different stages, we use \algname{CAB1} in a clever way: besides partitioning the arm space, we add summaries of previous regimes by allowing the algorithm to play according to the empirical distributions of past plays. This is formally described in Subsection~\ref{subsec:and}. 

A salient difference with all previous approaches is that we zoom out by using fewer and fewer arms. To our knowledge, this is unique, as all other algorithms for bandits zoom in in a way that crucially depends on the regularity parameters. Another important feature of our analysis is that we adapt both to the Hölder exponent $\alpha$ and to the Lipschitz constant  $L$. On a technical level, this is thanks to the fact that we do not explicitly choose a grid of regularity parameters, which means that we implicitly handle all values $(L, \alpha)$ simultaneously.

We first give a regret bound in the known horizon case (Subsection~\ref{subsec:and}), then we provide an anytime version and we show that they match the lower bounds of adaptation (Subsection~\ref{subsec:discussion}). Finally Section~\ref{sec:proof} provides the proof of our main regret bound.

\section{Setup, preliminary discussion}\label{sec:setup}

\subsection{Notation and known results}\label{subsec:setting}

Let us reintroduce briefly the standard bandit terminology. We consider the arm space $\mathcal X = [0,1]$. The environment sets a reward function $f : \mathcal X \to [0,1]$. At each time step $t$, the player chooses an arm $X_t \in \mathcal X$, and the environment then displays a reward $Y_t$ such that $\E[Y_t \mid X_t] = f(X_t)$, independently from the past. We assume that the variables $Y_t - f(X_t)$ are ($1/4$)-subgaussian conditionnally on $X_t$; this is satisfied if the payoffs are bounded in $[0,1]$ by Hoeffding's lemma.

The objective of the player is to find a strategy that minimizes her \emph{expected cumulative (pseudo-)regret.} If $M(f)$ denotes the maximum value of $f$, the regret at time $T$ is defined as
\begin{equation}
	\overline R_T = T M(f) - \E \! \left[ \sum_{t= 1}^{T}Y_t \right] = TM(f) - \E \!\left[\sum_{t= 1}^{T} f(X_t)  \right] \, .
\end{equation}

In this paper, we assume that the function $f$ satisfies a Hölder assumption around one of its maxima:
\begin{definition}For $\alpha>0$ and $L > 0$, we denote by $\mathcal H(L, \alpha)$ the set of functions that satisfy 
\begin{equation}\label{eq:holdercondition}
\exists \,x^\star \in  [0,1] \;\; \text{s.t.} \;\;  f(x^\star) = M(f) \;\; \text{and} \;\; \forall \, x \in [0,1]\,  \quad   \abs{f(x^\star) - f(x)} \leq L\abs{x^\star - x}^\alpha \, .
\end{equation}
\end{definition}

We are interested in minimax rates of regret when the mean-payoff function $f$ belongs to these Hölder-type classes, i.e., the quantity
$
	 \displaystyle \inf_{\text{\tiny algorithms}} \sup_{f \in \mathcal H (L, \alpha)} \Rt  \, .
$

\paragraph{MOSS}Throughout this paper, we exploit discretization arguments and use a minimax optimal algorithm for finite-armed bandits: \algname{MOSS}, from \citet{audibert2009Minimax}. When run for $T$ rounds on a $K$-armed bandit problem with $(1/4)$-subgaussian noise, and when $T \geq K$, its regret is upper-bounded by $18 \sqrt{KT}$ (the improved constant is from \citet{garivier2018KLUCBswitch}).

\paragraph{Non-adaptive minimax rates} When the regularity is given to the player, for any $\alpha, L$ and $T$:
\begin{equation}
0.001 \, L^{1 / (2\alpha + 1)} T^{(\alpha + 1)/(2\alpha+ 1)} \leq \inf_{\text{\tiny algorithms}} \sup_{f \in \mathcal H (L, \alpha)} \Rt  \leq 28 \, L^{1 / (2\alpha + 1)} T^{(\alpha + 1)/(2\alpha+ 1)}\, .
\end{equation}
This is well-known since \citet{kleinberg2004Nearly}. For completeness, we recall how to derive the upper bound in Section~\ref{subsec:cab1.1}, and the lower bound in Section~\ref{subsec:lowerbounds}.

\subsection{Lower bounds: adaptation \emph{at usual rates} is not possible}\label{subsec:lowerbounds}

\citet{locatelli2018adaptivity} prove a version of the following theorem; see our reshuffled and slightly improved proof in Appendix~\ref{app:lowbounds}.
\begin{theorem*}[Variation on Th.3 from \cite{locatelli2018adaptivity}]
	Let $B > 0$ be a positive number. Let $ \alpha, \gamma >0$  and $L , \ell > 0 $ be regularity parameters that satify $\alpha \leq \gamma$ and $L \geq \ell$.
	
	Assume moreover that	 $2^{-3} \, 12^\alpha  B^{-1} \leq L \leq \ell^{1 + \alpha} \,  T^{\alpha / 2} \, 2^{(1+\alpha)(8 - 2 \gamma)} $. If an algorithm is such that
	$
	\sup_{f \in \mathcal H(\ell, \gamma)} \Rt \leq B \, , 
	$
	then the regret of this algorithm is lower bounded on $\mathcal H(L, \alpha)$: 
	\begin{equation}\label{eq:adaptivelow}
	\sup_{f \in \mathcal H(L, \alpha)} \Rt \geq 2^{-10} \,  T L^{1 / (\alpha + 1)} B^{- \alpha /(\alpha + 1)} \, .
	\end{equation}
\end{theorem*}
\begin{remark*}[Bibliographical note]
	\citet{locatelli2018adaptivity} consider a more general setting where additional margin conditions are exploited. In our setting, we slightly improve their result by dealing with the dependence on the Lipschitz constant, and by removing a requirement on $B$.
	
	In a different context, \citet{krishnamurthy2019Contextual} show a variation of this bound where the Lipschitz constant is considered, but only in the case where $\alpha = \gamma =1$, for $\ell = 1$ and $L \geq 1$.
\end{remark*}

As explained in \citet{locatelli2018adaptivity} this forbids adaptation at the usual minimax rates over two regularity classes; we recall how in the paragraph that follows Theorem~\ref{thm:admrates}. However this is not the end of the story, as one naturally wonders what is the best the player can do. 

To further investigate this question, we discuss it asymptotically by considering the rates at which the minimax regret goes to infinity, therefore focusing on the dependence on $T$.  Our main results are completely nonasymptotic, yet we feel the asymptotic analysis of optimality is clearer.

\begin{definition}
	Let $\theta : [0, 1] \to [0, 1]$ denote a nonincreasing function. We say an algorithm \emph{achieves adaptive rates $\theta$} if
	\begin{equation*}
	\forall \, \epsilon > 0\, , \; \forall \, \alpha, \, L\,  > 0 \;, \quad \limsup_{T \to \infty} \frac{\sup_{f \in \mathcal H(L, \alpha)} \Rt}{T^{\theta(\alpha)+\epsilon} }
	< + \infty \, .
	\end{equation*}
\end{definition}
We include the $\epsilon$ in the definition in order to neglect the potential logarithmic factors.

As rate functions are not always comparable for pointwise order, the good notion of optimality is the standard statistical notion of \emph{admissibility} (akin to ``Pareto optimality'' for game-theorists).
\begin{definition}
	A rate function is said to be \emph{admissible} if it is achieved by some algorithm, and if no other algorithm achieves stricly smaller rates for pointwise order. An algorithm is admissible if it achieves an admissible rate function.
\end{definition}
We recall that a function $\theta'$ is stricly smaller than $\theta$ for pointwise order if $\theta'(\alpha) \leq \theta(\alpha)$ for all $\alpha$ and $\theta'(\alpha_0) < \theta(\alpha_0)$ for at least one value of $\alpha_0$.

It turns out we can fully characterize the admissible rate functions by inspecting the lower bounds~\eqref{eq:adaptivelow}. 
\begin{theorem}\label{thm:admrates}
	The admissible rate functions are exactly the family
	\begin{equation}\label{eq:thetam}
	\theta_m :\alpha \mapsto \max\left( m \, , \; 1 - m \frac{\alpha}{\alpha + 1}\right)  \, ,\quad  m \in [1 / 2, 1] \, .
	\end{equation}
\end{theorem}
This theorem contains two assertions. The lower bound side states that no smaller rate function may be achieved by any algorithm. This side is derived from an asymptotic rewording of lower bound~\eqref{eq:adaptivelow}, see Proposition~\ref{prop:rateslow} stated below (proofs are in Appendix~\ref{app:proofs}). The second statement is that the $\theta_m$'s are indeed achieved by an algorithm, which is the subject of Section~\ref{subsec:and}.

Figure~\ref{fig:curves} illustrates how these admissible rates compare to each other, and to the usual minimax rates.
\begin{figure}[H]
	\center
	\hspace{-1cm}
	\includegraphics[width=10cm]{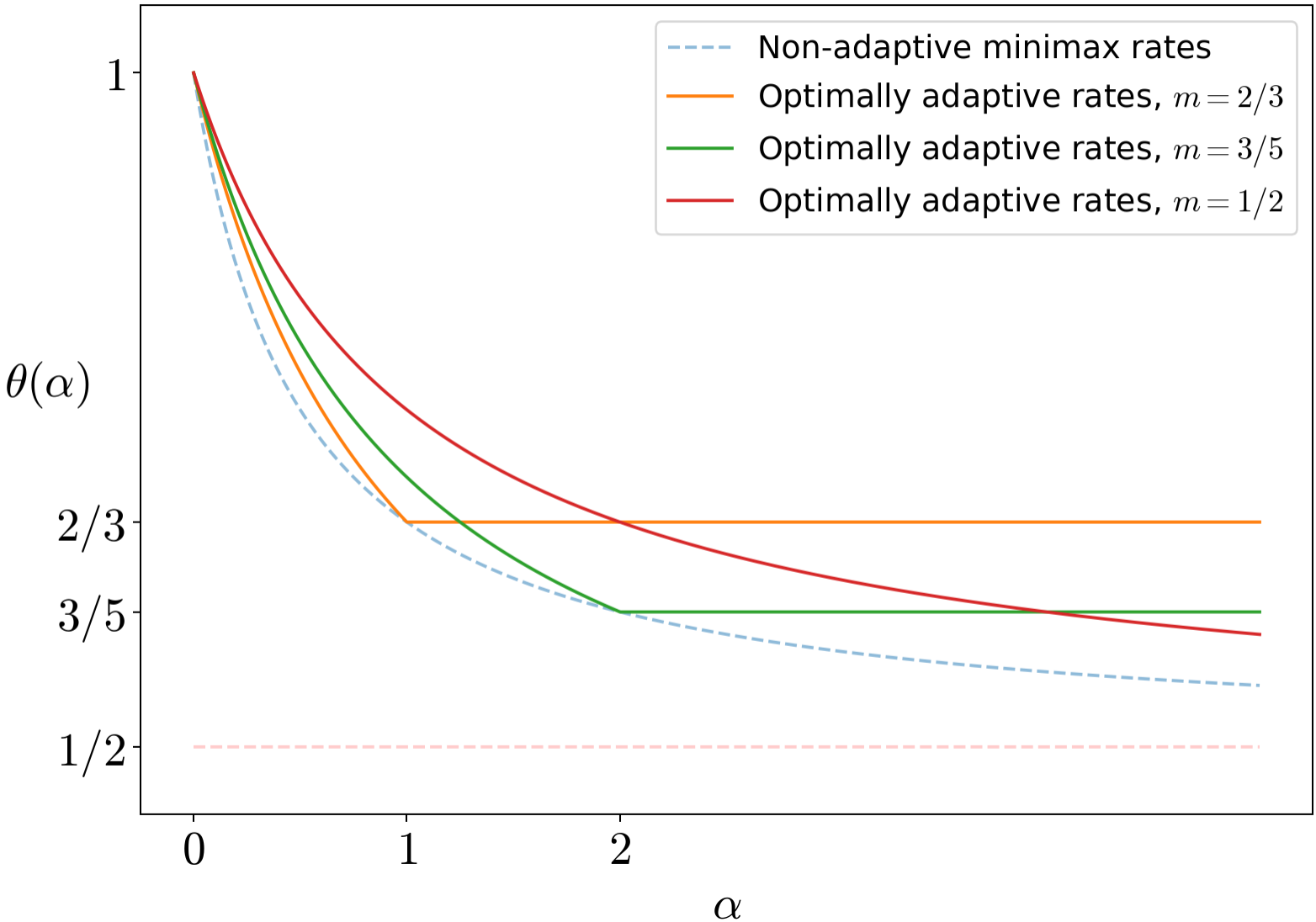}
	\caption{The lower bounds on adaptive rates: plots of the admissible rate functions $\alpha \mapsto \theta_m(\alpha)$.
		If an algorithm has regret of order $\mathcal O \big(T^{\theta(\alpha)}\big)$, then $\theta$ is everywhere above one of these curves.}
		\label{fig:curves}
\end{figure}
In particular, we see that reaching the nonadaptive minimax rates for multiple values of $\alpha$ is impossible. Moreover, at $m = (\gamma + 1)/(2 \gamma + 1)$, we have $\theta_m(\gamma) = (\gamma + 1)/(2 \gamma+ 1)$, which is the usual minimax rate~\eqref{eq:1} when $\gamma$ is known. This yields an alternative parameterization of the family $\theta_m$: one may choose to parameterize the functions either by their value at infinity $m\in [1/2, 1]$, or by the only point $\gamma \in [0 , + \infty]$ at which they coincide with the usual minimax rates function \eqref{eq:1}.
\begin{proposition}\label{prop:rateslow}
	Assume an algorithm achieves adaptive rates $\theta$. Then $\theta$ satisfies the functional inequation
	\begin{equation}\label{eq:rateslow1}
	\forall\, \gamma > 0\, , \quad \forall \, \alpha \leq \gamma \, , \quad  \theta(\alpha) \geq 1 - \theta(\gamma) \frac{\alpha}{\alpha + 1} \, .
	\end{equation}
\end{proposition}

\subsection{Yet can we adapt in some way?}

We have described in~\eqref{eq:thetam} the minimal rate functions that are compatible with the lower bounds of adaptation: no algorithm can enjoy uniformly better rates. Of course, at this point, the next natural question is whether any of these adaptive rate functions may indeed be reached by an algorithm.

All previous algorithms for continuum-armed bandits require the regularity as an input in some way (see the literature discussion in Section~\ref{subsec:relworks}). Such algorithms are flawed: if the true regularity is underestimated then we only recover the guarantees that correspond to the smaller regularity, which is often far worse than the lower bounds of Theorem~\ref{thm:admrates}. More dramatically, if the true regularity is overestimated, then, a priori, no guarantees hold at all.

We prove that all these rate functions may be achieved by a new algorithm. More precisely, if the player wishes to reach one of the lower bounds $\theta_m$, she may select a value of the input accordingly and match the chosen $\theta_m$. This is our main contribution and is described in the next section.

\section{An admissible adaptive algorithm and its analysis }\label{sec:alg}
We discuss in Section~\ref{subsec:cab1.1} how the well-known CAB1 algorithm can be generalized for our purpose. In Section~\ref{subsec:and} we describe our algorithm and the main upper bound on its regret. Section~\ref{subsec:discussion} is devoted to the anytime version of the algorithm and to a discussion on optimality.

\subsection{An abstract version of CAB1 as a building block towards adaptation}
\label{subsec:cab1.1}
We describe a generalization of the \algname{CAB1} algorithm from \citet{kleinberg2004Nearly}, where we include arbitrary measures in the discretization. Although this extension is straightforward, we detail it as we will use this algorithm repeatedly further in this paper. In the original \algname{CAB1}, the space of arms is discretized into a partition of $K$ subsets, and an algorithm for finite-armed bandits plays on the $K$ midpoints of the sets. \citet{auer2007improved} replace the midpoints by a random point uniformly chosen in the subset.

We introduce a generic version of this algorithm we call \algname{CAB1.1}. We consider $K$ arbitrary probability distributions over $\mathcal X$, which we denote by $(\pi_i)_{1\leq i \leq K}$. Denote also by $\pi(f)$ the expectation of $f(X)$ when $X\sim \pi$. At each time step, the decision maker chooses one distribution, $\pi_{I_t}$, and plays an arm picked according to that distribution. By the tower rule, she receives a reward such that
\begin{equation*}
\E[Y_t \mid I_t] 
= \E[f(X_t) \mid I_t] 
= \pi_{I_t}(f) \, .
\end{equation*}
As the player uses a finite-arm algorithm $\mathcal A$ to select $I_t$, the regret she suffers can be decomposed as the sum of two terms (denoting by $\tilde R_T$ the expected regret of the finite-armed algorithm):
\begin{equation}\label{eq:decompositionregret}
\Rt
= T \big(M(f) - \max_{i = 1, \dots, K} \pi_i(f)\big) + \tilde R_T\big(\big(\pi_i(f)\big)_{1\leq i \leq K} ; \mathcal A\big)\, .
\end{equation}
This identity is central to the construction of our algorithm. Using terminology from \citet{auer2007improved}, the first term measures an \emph{approximation error} of the maximum of $f$, and the other the actual \emph{cost of learning} in the approximate problem. Parameters are chosen to balance these two sources of error.\vspace{-0.15cm}
\begin{algorithm}[H]
	\begin{algorithmic}[1]
		\caption{CAB1.1 (\underline{C}ontinuum-\underline{A}rmed \underline{B}andit, adapted from \citet{kleinberg2004Nearly})}
		\label{alg:discretize}
		\STATE \textbf{Input}: $T$ the time horizon, $K$ probability measures over $\mathcal X$ denoted by $\pi_1, \dots, \pi_K$, discrete $K$-armed bandit algorithm $\mathcal A$
		\FOR{$t=1,2,\ldots, T$ }
		\STATE Define $I_t$ the arm in $\{1, \dots, K\}$ recommended by $\mathcal A$
		\STATE Play $X_t \in \mathcal X$ drawn according to $\pi_{I_t}$, and receive $Y_t$  such that $\E[Y_t \vert X_t] = f(X_t)$
		\STATE Give $Y_t$ as input to $\mathcal A$ corresponding to $I_t$
		\ENDFOR
	\end{algorithmic}
\end{algorithm}\vspace{-0.15cm}
The canonical example is that for which the space of arms is cut into a partition. Denote by $\Disc(K)$ the family of the uniform measures over the intervals $[(i-1) /K, i / K]$ for $1 \leq i \leq K$. We state this results (and prove it in Appendix~\ref{app:proofscab}) to recall the non-adaptive minimax bound \eqref{eq:1}.
\begin{proposition}\label{prop:discretizebound}
	Let $\alpha >0$ and $L > 1 / \sqrt T $ be regularity parameters, and define the number of discrete arms $K^\star = \min \big( \ceil{L^{2 / (2\alpha +1)} T^{1/(2\alpha + 1)}}, T \big)$. 
	Algorithm \algname{CAB1.1} run with the uniform discretization $\Disc(K^\star)$ and $\mathcal A = $\algname{MOSS} enjoys the bound  \quad
	$\displaystyle
	\sup_{f \in \mathcal H(L, \alpha)} \Rt
	\leq 28 \,   L^{1 / (2\alpha + 1)} \,  T^{(\alpha + 1)/(2\alpha + 1)} \, .
	$
\end{proposition}

\subsection{Memorize past plays, Discretize the arm space, and Zoom Out: the MeDZO algorithm}\label{subsec:and}

To achieve adaptation, we combine two tricks: \emph{going from fine to coarser discretizations} while \emph{keeping a summary of past plays in memory}. 


Our algorithm works in successive regimes. At each time regime $i$, we reset the algorithm and start over a new regime of length double the previous one ($\Delta T_i = 2^{p+i}$), and with fewer discrete arms ($K_i = 2^{p+2-i}$). While doing this, we keep in memory the previous plays: in addition to the uniform distributions over the subsets of partitions, we include the empirical measures $\what \nu_j$ of the actions played in the past regimes, for $j < i$.
\vspace{-0.5cm}

\noindent\begin{minipage}{\textwidth}
	\renewcommand\footnoterule{}
	\begin{algorithm}[H]
		\begin{algorithmic}[1]
			\caption{MeDZO (\underline{Me}morize, \underline{D}iscretize, \underline{Z}oom \underline{O}ut)}
			\label{alg:adaptivediscretization}
			\STATE \textbf{Input}: parameter $ B $, time horizon $T$
			\STATE \textbf{Set}: $p = \ceil{\log_2 B}$, $K_i= 2^{p+2-i}$ and $\Delta T_i = 2^{p+i}$
			\FOR{$i = 1, \ldots, p$ }
			\STATE For $\Delta T_i$ rounds, run algorithm \algname{CAB1.1} with the uniform discretization in $K_i$ pieces \emph{and} the empirical measures of the previous plays $\what \nu_j$ for $j < i$; use \algname{MOSS} as the discrete algorithm.\footnote{No $\what \nu$ is used for $i=0$}
			\STATE \textbf{Set}: $\what{\nu}_{i}$ the empirical measure of the plays during regime $i$.
			\ENDFOR
		\end{algorithmic}
	\end{algorithm}
\end{minipage}
Appendix~\ref{app:illustration} provides a figure illustrating the behavior of the algorithm.

Our construction is based on the following remark. Consider the approximation error suffered during regime $i$. Denoting the by $\Pi_i$ the set of measures given to the player during regime $i$, that is, the uniform measures over the regular $K_i$-partition and the empirical measures of arms played during the regimes $j < i$, the approximation error is bounded as follows:
\begin{equation}\label{eq:descriptionapprox}
\! \Delta T_i \!  \left( \!  M(f)  - \E \Big[\! \max_{ \pi \in \Pi_i} \pi(f)\Big]\right) \! 
\leq \Delta T_i \, \big(M(f) - \E[\what \nu_j(f)] \big)
=  \frac{\Delta T_i}{\Delta T_j} \! \sum_{t \in  \text{\tiny Regime }j} \! \! \! \! \big( M(f) - \E[f(X_t)]\big)
\end{equation}
and this bound is proportional to the regret suffered during regime $j$. This means that even though we zoom out by using fewer arms, we can make sure that the average approximation error in regime $i$ is less than the regret previously suffered. Moreover, the first discretizations are fine enough to ensure a small regret in the first regimes, thanks to the Hölder property. This argument is formalized in the proof (Lemma~\ref{lem:approx}), and shows that MeDZO maintains a balance between approximation and cost of learning that yields optimal regret.

A surprising fact here is that we go from finer to coarser discretizations during the different phases. Thus, paradoxically, \emph{the algorithm zooms out as time passes}. Note also that although this regime-based approach is reminiscent of the doubling trick, there is an essential difference in that information is carried between the regimes via the distribution of the previous plays.

We first state our central result, a generic bound that holds for any input parameter $B$. We discuss the optimality of these adaptive bounds in the next subsection.

\begin{theorem}\label{th:and}
	Algorithm \algref{alg:adaptivediscretization} run with the knowledge of $T$ and input $B \geq \sqrt T $ enjoys the following guarantee: for all $\alpha>0$ and $L >0$,
	\begin{equation}\label{eq:totalbound}
	\sup_{f \in \mathcal H (L, \alpha)} \Rt \leq  412 \, (\log_2 B)^{3/2}  \max \big( B , \, T L^{1/(\alpha + 1) }B^{-\alpha / (\alpha + 1)} \, \big) \, .
	\end{equation}
\end{theorem}
We provide some illustrative numerical experiments in Appendix~\ref{app:experiments}, comparing the results of MeDZO with other non-adaptive algorithms.

\subsection{Discussion: anytime version and admissibility}\label{subsec:discussion}
\paragraph{Anytime version via the doubling trick}
The dependence of Algorithm~\ref{alg:adaptivediscretization} on the parameter $B$ makes it horizon-dependent. We use the doubling trick to build an anytime version of the algorithm. At each new doubling-trick regime, we input a value of $B$ that depends on the length of the $k$-th regime. If it is of length $T^{(k)}$, one typically thinks of $B_k = (T^{(k)})^m$ for some exponent $m$. In that case, we get the following bound ---see the proof and description of the algorithm in Appendix~\ref{app:anytimeproof}.

\begin{corollary}[Doubling trick version]\label{prop:doublingtrick}
	Choose $m \in [1/2, 1]$. The doubling-trick version of MeDZO, run with $m$ as sole input (and without the knowledge of T) ensures that for all regularity parameters $\alpha >0$ and $L > 0$ and for $T \geq 1$
	\begin{equation*}
	\! \sup_{f \in \mathcal H (L, \alpha)} \Rt \leq  4000 (\log_2 T^m )^{3/2}  \max \big( T^m , T L^{1/(\alpha + 1) }(T^m)^{-\alpha / (\alpha + 1)} \big) \!
	= {\mathcal O} \big( (\log T)^{3/2}\, T^{\theta_m(\alpha)}\big) .
	\end{equation*}
\end{corollary}

\paragraph{Admissibility of Algorithm~\ref{alg:adaptivediscretization}} The next result is a direct consequence of Corollary~\ref{prop:doublingtrick}. This echoes the discussion following Theorem~\ref{thm:admrates}, and shows that for any input parameter $m$, the anytime version of  \algname{MeDZO} cannot be improved uniformly for all $\alpha$.
\begin{corollary}\label{cor:medzorates}
	For any $m \in [1/2, 1]$, the doubling trick version of MeDZO (see App.~\ref{app:anytimeproof}) with input $m$ achieves rate function $\theta_m$, and is therefore admissible.
\end{corollary} 

\subsection{About the remaining parameter: the $B = \sqrt T$ case}

Tuning the value of $B$ amounts to selecting one of the minimal curves in Figure~\ref{fig:curves}. Therefore this parameter is a feature of the algorithm, as it allows the player to choose between the possible optimal behaviors. The tuning of this parameter is an unavoidable choice for the player to make.

The next example illustrates well the performance of \algname{MeDZO}, as it is easily comparable to the usual minimax bounds. Looking at Figure~\ref{fig:curves}, this choice corresponds to $m = 1/2$, i.e., the only choice of parameter that reaches the usual minimax rates as $\alpha \to \infty$. In other words, if the players wishes to ensure that her regret on very regular functions is of order $\sqrt T$, then she has to pay a polynomial cost of adaptation for not knowing $\alpha$ and that price is exactly the ratio between~\eqref{eq:1} and~\eqref{eq:2}.
\begin{corollary}\label{cor:sqrtT}
	Set a horizon $T$ and run Algorithm \algref{alg:adaptivediscretization} with $B = \sqrt T$. Then for $\alpha >0$ and $L > 1 / \sqrt T$,
	\begin{equation}\label{eq:sqrtbound*}
	\sup_{f \in \mathcal H (L, \alpha)} \Rt \leq  146  \, (\log_2 T )^{3/2}  \, L^{1 / (\alpha + 1)}\, T^{(\alpha +2)/(2\alpha + 2)} \, . 
	\end{equation}
\end{corollary}
This is straightforward from Theorem~\ref{th:and}, since the inequality $B = \sqrt T \leq  T L^{1/(\alpha + 1)} \sqrt{T}^{-\alpha / (\alpha+1)}$ holds  whenever $L \geq 1/ \sqrt T$. An anytime version of this result can be obtained from Corollary~\ref{prop:doublingtrick}.
\vspace{-0.25cm}
\section{Proof of Theorem~\ref{th:and}}\label{sec:proof}
\vspace{-0.25cm}

\begin{proof}[Full proof of Theorem~\ref{th:and}]
	Let $\mathcal F_t = \sigma(I_1, X_1, Y_1 ,\dots,  I_t, X_t, Y_t)$ be the $\sigma$-algebra corresponding to the information available at the end of round $t$. Define also the transition times $T_i = \sum_{j = 1}^{i} \Delta T_j$ with the convention $T_0 = 0$.  Let us first verify that $T$ is smaller than the total length of the regimes. By definition of $p$, we have $B \leq 2^p < 2 B$. Thus $T_p = 2^{p+1}(2^{p}- 1)\geq 2B(B-1) > B^2 > T$, and the algorithm is indeed well-defined up to time $T$.
	
	Consider the regret suffered during the $i$-th regime
	$
	\overline R_{T_{i-1}, T_i} := \Delta T_i \, M(f) - \sum_{t = T_{i-1} + 1}^{T_i} \E \big[  f(X_t) \big] \,.
	$
	We bound this quantity thanks to the decomposition \eqref{eq:decompositionregret}, by first conditioning on the past up to time $T_{i-1}$. Since there are $K_i +i$ discrete actions, the regret bound on MOSS ensures that
	\begin{equation}\label{eq:regretperiodupbound}
	\E \!\left[ \sum_{t = T_{i-1} + 1}^{T_i}  \big( M(f) - f(X_t)\big) \biggm\vert \mathcal F_{T_{i-1}} \right] 
	\leq \Delta T_i \,\big( M(f) - M_i^\star \big) + 18 \sqrt{(K_i +i)\Delta T_i} 
	\end{equation}
	where
	$	
	M_i^\star 
	= \max \{ \pi_j^{(i)}(f) \mid \pi_{j}^{(i)} \in \Disc(K_i) \} \cup \big\{\what{\nu}_\ell(f)  \mid  \ell = 0, \dots, i-1  \big\} \,.
	$
	Notice that this bound holds even though $M_i^\star$ is a random variable, as the algorithm is completely reset, and the measures $(\what{\nu}_{j})_{j<i}$ are fixed at time $T_{i-1}+1$ (i.e., they are $\mathcal F_{T_{i-1}}$-measurable). Integrating once more, we obtain
	\begin{equation}\label{eq:app+cl}
	\overline R_{T_{i-1}, T_i} \leq \Delta T_i \,\big( M(f) - \E[M_i^\star] \big) + 18 \sqrt{(K_i +i)\Delta T_i} \, .
	\end{equation}
	
	\paragraph{Bounding the cost of learning.}By definition of $K_i$ and $\Delta T_i$, we have
	$
	K_i \Delta T_i = 2^{2p +2} \leq 16B^2 \, . 
	$
	Therefore, since $p$ and $K_i$ are integers greater than $1$, using $a+b - 1 \leq ab$ for positive integers,
	\begin{equation}\label{eq:alldiscreteerror}
	\sqrt{(K_i + i )\Delta T_i} 
	\leq \sqrt{ (K_i + p -1) \Delta T_i} 
	\leq  \sqrt{ p K_i \Delta T_i} 
	\leq 4 \sqrt p B \, .
	\end{equation}
	
	\paragraph{Bounding the approximation error.} The key ingredient for this part is the following fact, that synthetizes the benefits of our construction as hinted in \eqref{eq:descriptionapprox} and the surrounding discussion.
	
	\begin{lemma}\label{lem:approx}
		The total approximation error of MeDZO in regime $i$ is controlled by the Hölder bound on the grid of mesh size $1/K_i$, and by the regret suffered during the previous regimes,
		\begin{equation}\label{eq:boundapprox}
		\Delta T_i \,\big(M(f) - \E[M_i^\star] \big)  \leq \Delta T_i \,  \min \! \left( L\, \frac{1}{K_i^\alpha} , \min_{j < i}\,    \bigg(  \frac{\overline R_{T_{j-1}, T_j}}{\Delta T_j} \bigg)\right)
		\end{equation}
	\end{lemma}
	\begin{proof}
	This derives easily from the construction of the algorithm, i.e., from the definition of $M_i^\star$. 
	Considering an interval in the regular $K_i$-partition that contains a maximum of $f$, by the Hölder property, 
	$
	M(f) - M_i^\star \leq  \ifrac{L}{K_i^\alpha} \, .
	$
	For the second minimum, as described in Eq.~\eqref{eq:descriptionapprox}, for $j < i$,
	\begin{equation*}
	M(f) - M_i^\star 
	\leq M(f) - \what{\nu}_{j}(f) 
	= \frac{1}{\Delta T_{j}} \sum_{t = T_{j -1}+1}^{T_{j}} \big( M(f) - f(X_t) \big) \,.
	\end{equation*}
	Taking an expectation, $\overline R_{T_{j-1}, T_j}$ appears, and we conclude by taking the minimum over $j$.
	\end{proof}
	
	 Remember that since $K_i \Delta T_i = 2^{2p + 2} $, we have $L\, \ifrac{\Delta T_i}{K_i^\alpha} = L \ifrac{2^{2p+2}}{K_i^{1 + \alpha}} $. Therefore, the first bound on the approximation error in \eqref{eq:boundapprox} increases with $i$, as $K_i$ decreases with $i$. 
	Denote by $i_0$ the last time regime $i$ for which
	\begin{equation}\label{eq:i_0}
	L \frac{ \Delta T_{i_0}}{K_{i_0}^\alpha} \leq B \, .
	\end{equation}
	If this is never satisfied, i.e., not even for $i = 1$, then $L 2^{p+1} /2^{\alpha (p+1)} > B$ which yields, using $B \leq2^p \leq 2B$, that $4LB \geq 2^{\alpha + 1}B^\alpha B$ and then $L > B^\alpha / 2$. In that case, $L^{1 / (\alpha + 1)}B^{-\alpha/(\alpha + 1)} \geq 1$ and the total regret bound \eqref{eq:totalbound} is true as it is weaker than the trivial bound $R_T \leq T$.
	
	Hence we may assume that $i_0 \geq 1$ is well defined. By comparing $i$ to $i_0$, we now show the inequality
	\begin{equation}\label{eq:approxsum}
	\sum_{i = 1}^p \Delta T_i \big(M(f) - \E[M_i^\star]\big)  
	\leq\sum_{i = 1}^{i_0} B +  \sum_{i = i_0 + 1}^{p} 2(1 + 72 \sqrt p )\Delta T_{i} \,   L^{1 / (\alpha + 1)} B^{-\alpha /(\alpha + 1)} \, .
	\end{equation}	
	
	For all $i \leq i_0$ the approximation error is smaller than the first argument of the minimum in~\eqref{eq:boundapprox}, and this term is smaller than $B$. Therefore
	$
	\Delta T_i \big(M(f) - \E[M_i^\star]\big)   \leq B \, .
	$
	In particular, this together with \eqref{eq:app+cl} and \eqref{eq:alldiscreteerror} implies that the total regret suffered during regime $i_0$ is $\overline R_{T_{i_0-1}, T_{i_0 }} \leq  (1 + 72 \sqrt p) B$.
	
	For the later time regimes $i_0 <i \leq p$, we use the fact that preceding empirical measures were kept as discrete actions, and in particular the one of the $i_0$-th regime: \eqref{eq:boundapprox} instantiated with $j = i_0$ yields
	\begin{equation}\label{eq:approxbound}
	\Delta T_{i}\, \big(M(f) - \E[M_i^\star]\big) 
	\leq \Delta T_{i} \frac{\overline R_{T_{i_0-1}, T_{i_0}}}{\Delta T_{i_0}} 
	\leq \big(1 + 72 \sqrt p \big) \Delta T_i \frac{B}{\Delta T_{i_0}}  \, .
	\end{equation}
	
	Solving equations $L \Delta T_{i_0} / K_{i_0}^\alpha \approx B \approx 4 \sqrt{\Delta T_{i_0} K_{i_0}} $, we get
	$
	\ifrac{B}{\Delta T_{i_0}} \leq 2\,L^{1 / (\alpha + 1)} B^{- \alpha /(\alpha +1)} \, ,
	$
	(details are in Appendix~\ref{app:proofcalc}).
	Therefore for $i_0 < i \leq p $, using \eqref{eq:approxbound}, 
	\begin{equation*}
	\Delta T_i \big(M(f) - \E[M_i^\star]\big)  \leq 2 (1 + 72 \sqrt p) \, \Delta T_i \, \,   L^{1 / (\alpha + 1)} B^{-\alpha /(\alpha + 1)}\, ,
	\end{equation*}
	and we obtain \eqref{eq:approxsum} by summing over $i$.
	
	\paragraph{Conclusion}
	We conclude with some crude boundings. First, as $i_0 \leq p$ and the sum of the $\Delta T_i$'s is smaller than $T$, the total approximation error is less than
	$
		p B + 2 (1 + 72 \sqrt p ) T L^{1 / (\alpha + 1)} B^{- \alpha / (\alpha + 1)}.
	$
	Let us include the cost of learning, which is smaller than $72 p \sqrt p B$ and conclude, using $a + b \leq \max(a, b)$
	\begin{equation}\label{eq:almostfinalbound}
	\begin{split}
	\overline R_T 
	&\leq 2(1 + 72 \sqrt p )T  L^{1 / (\alpha + 1)} B^{-\alpha /(\alpha + 1)}+ p  B +  72 p^{3/2} B  \\
	& = 2(1 + 72 \sqrt p )T  L^{1 / (\alpha + 1)} B^{-\alpha /(\alpha + 1)} +  p(1 + 72 \sqrt{p} )B \\
	&\leq \Big(2(1 + 72 \sqrt p ) + p \big(1 + 72 \sqrt p \big) \Big)\max\big(B,\, T  L^{1 / (\alpha + 1)} B^{-\alpha /(\alpha + 1)}\, \big)
	\end{split}
	\end{equation}
	from which the desired bound follows, using $1 \leq p$, and $p \leq 2\log_2 B$ and $4(1+ 72 \sqrt 2) \leq 412$.
\end{proof}
\vspace{-0.27cm}

\section{Further considerations}\label{sec:further}
\vspace{-0.2cm}

\paragraph{Local regularity assumption}
Theorem~\ref{th:and} holds under a relaxed smoothness assumption, namely that the function satisfies the Hölder condition only in a small cell containing the maximum. By looking carefully at the proof, we observe that the condition is only required up to the $i_0$-th epoch (defined in \eqref{eq:i_0}), at which the size of the cells in the discretization is of order $1 / K_{i_0} \approx (LB)^{-1/(1 + \alpha)}$. 
Therefore we only need condition~\eqref{eq:holdercondition} to be satisfied for points $x$ in an interval of size $(LB)^{-1/(1 + \alpha)}$ around the maximum.

\paragraph{Higher dimension}
Our results can be generalized  to functions $[0, 1]^d \to [0,1]$ that are $\lVert \cdot \rVert_{\infty}$-Hölder. For MeDZO to be well-defined, take $K_i = 2^{d(p+2-i)}$ and $\Delta T_i = 2^{d(p+i)}$, with $p \approx (\log B) / d$.
The bounds are similar to their one-dimensional counterparts, up to replacing $\alpha$ by $\alpha / d$ in the exponents, but the constants are deteriorated by a factor that is exponential in $d$. The bound in Theorem~\ref{th:and} changes to $\max\big( B, L^{d /(\alpha + d)}T B^{-\alpha / (\alpha + d)} )$.


%% file: final_bandit_adaptation_supp.tex

\onecolumn
\begin{center}
  {\Large\bf Supplementary Material for \\
  	"Polynomial Cost of Adaptation for $\mathcal X$-Armed Bandits"}
\end{center}          

\renewcommand{\thesubsection}{\Alph{subsection}}
\setcounter{subsection}{0}

\subsection{Omitted proofs}\label{app:proofs}

\subsubsection{Proposition \ref{prop:discretizebound}: Regret bound for non-adaptive CAB1.1}
\label{app:proofscab}
This proof is a straightforward application of the Hölder bound and of the bound of MOSS, together with the approximation/cost of learning decomposition of the regret. Some extra care is needed to handle the boundary cases.
\begin{proof}[Proof of Proposition \ref{prop:discretizebound}]
	Choose $f \in \mathcal H(L, \alpha)$. Let us denote by $i^\star$ an integer such that there exists an optimal arm $x^\star$ in the interval $\big[(i^\star-1)/K^{\star}, i^\star / K^{\star} \big]$.  By the Hölder assumption
	\begin{equation*}
	\frac{1}{K^\star}\int_{(i^\star - 1)/K^\star}^{i^\star/K^\star} \big(f(x^\star) - f(x)\big) \dint x 
	\leq L \left(\frac{1}{K^\star}\right)^\alpha \, ,
	\end{equation*}
	and this upper bounds the approximation error of the discretization.
	Moreover, since $T \geq K^\star $, the cost of learning is smaller than $18 \sqrt{K^\star T}$. Thus by \eqref{eq:decompositionregret}
	\begin{equation*}
	\Rt
	\leq T L \left(\frac{1}{K^\star}\right)^\alpha + 18 \sqrt{K^\star T}\, .
	\end{equation*}
	$K^\star$ was chosen to minimize this quantity. We distinguish cases depending on the value of $K^\star$.
	
	If $1 < K^\star < T$, then $L^{2 / (2\alpha + 1)} T^{1 / (2\alpha + 1)} \leq K^\star \leq 2 L^{2 / (2\alpha + 1)} T^{1 / (2\alpha + 1)} $ (the bound $\ceil{x} \leq 2x$, which is valid when $x \geq 1$, is more practical to handle the multiplicative constants), we deduce the upper bound:
	\begin{equation*}
	\big(1  + 18 \sqrt 2\big) \, L^{1 / (2\alpha + 1)} T^{(\alpha+1) / (2\alpha + 1)}\, .
	\end{equation*}

	Since we assumed that $L > 1/\sqrt T$, we have always $K^\star > 1$. Therefore the last case to consider is if $K^\star = T$. Then
	$
	L^{2 / (2\alpha + 1)} T^{1 / (2\alpha + 1)} \geq T/2
	$
	and thus
	$
	L \geq 2^{-(2\alpha + 1)/2} \, T^\alpha  \,.
	$
	In this case
	$
	L^{1/(2\alpha +1 )} T^{(\alpha+ 1)/(2\alpha + 1)} \geq  (\sqrt 2 / 2) T
	$
	and the claimed bound is met since in that case, we have by a trivial bound $\Rt \leq T \leq \sqrt 2 \, L^{1/(2\alpha +1 )} T^{(\alpha+ 1)/(2\alpha + 1)} $.
\end{proof}

\subsubsection{Proposition \ref{prop:rateslow}: Lower bound on the adaptive rates}

\begin{proof}[Proof of Proposition \ref{prop:rateslow}]
	Choose $\alpha $, $\gamma$ such that $\alpha \leq \gamma$, and $\epsilon>0$. Set $L >0 $. There exist constants $c_1$ and $c_2$ (depending on $L, \alpha, \gamma$ and $\epsilon$) such that for $T$ large enough, 
	\begin{equation*}
	\sup_{f \in \mathcal H(L, \alpha)} \Rt \leq c_1 T^{\theta(\alpha) + \epsilon}
	\quad \text{ and } \quad
	\sup_{f \in \mathcal H(L, \gamma)} \Rt \leq c_2 T^{\theta(\gamma) + \epsilon} \, .
	\end{equation*}
	Moreover, for $T$ large enough, the assumptions for lower bound \eqref{eq:adaptivelow} hold. Hence applying the lower bound with $B = c_2 T^{\theta(\gamma) + \epsilon}$, for some constant $c_3$:
	\begin{equation*}
	c_1 T^{\theta(\alpha) + \epsilon} 
	\geq 0.0001\,  T\big(c_2 T^{\theta(\gamma) + \epsilon}\big)^{- \alpha / (\alpha + 1)}
	\geq 
	 c_3 \, T^{1 - \theta(\gamma ) \alpha / (\alpha + 1) - \epsilon \alpha / (\alpha + 1)}
	\end{equation*}
	Since the above inequality holds for any $T$ sufficiently large, this implies that for all $\epsilon > 0$
	\begin{equation*}
	\theta(\alpha) + \epsilon
	\geq 1 - \theta(\gamma)\frac{\alpha}{\alpha + 1} - \epsilon \frac{\alpha}{\alpha + 1},
	\end{equation*}
	which yields the desired result as $\epsilon \to 0$.
\end{proof}
\vspace{1cm}
\subsubsection{Theorem \ref{thm:admrates}: Admissible rate functions}
We prove here that all the admissible rate functions belong to the family $(\theta_m)$, by relying on Proposition \ref{prop:rateslow}. The proof is done through a careful inspection of the functional inequation defining the lower bound.
\begin{proof}[Proof of Theorem \ref{thm:admrates}]
	First of all, by Corollary~\ref{cor:medzorates}, the appropriately tuned MeDZO may achieve all the $\theta_m's$. Thus we are left to prove the lower bound side, i.e., that all the admissible rate functions belong to the family $\theta_m$.
	
	The best way to see this is to first notice that for $\theta$ nonincreasing and positive, the inequation in Proposition \ref{prop:rateslow} is equivalent to 
	\begin{equation}\label{eq:rateslow2}
	\forall\, \alpha > 0 \, ,   \quad  \theta(\alpha) \geq 1 - \theta(\infty) \frac{\alpha}{\alpha + 1} \, .
	\end{equation}
	Notice that taking $\gamma = + \infty$ is always valid in what follows, as $\theta$ is assumed to be nonincreasing and lower bounded by $1/2$.
	Now if $\theta$ satisfies \eqref{eq:rateslow1}, then it satisfies \eqref{eq:rateslow2} by taking $\gamma = + \infty$. For the converse, consider $\alpha \leq \gamma$, then $\theta(\gamma) \geq \theta(\infty)$, thus $1 - \theta(\infty) \alpha / (\alpha + 1) \geq 1 - \theta(\gamma) \alpha / (\alpha + 1)$. 
		
 	Now consider an admissible $\theta$. Since $\theta$ is achieved by some algorithm, by Proposition~\ref{prop:rateslow} and the remark above, it satisfies Eq.~\eqref{eq:rateslow2}. As $\theta$ is nonincreasing, and by Eq.~\eqref{eq:rateslow2}, we have $\theta(\alpha) \geq \theta(\infty)$ and $\theta(\alpha) \geq 1 - \theta(\infty)\alpha / (\alpha + 1)$.  In other words, $\theta \geq \theta_{m_{\theta}}$, where $m_\theta = \theta(\infty) \in [1/2, 1]$. By the admissibility of $\theta$, this implies that $\theta = \theta_{m_\theta}$.
\end{proof}

\subsubsection{Calculations in the proof of Theorem~\ref{th:and}} \label{app:proofcalc}
\begin{proof}[Details on \eqref{eq:approxsum}, in the proof of Theorem~\ref{th:and}]

	By definition of $i_0$, and since we assumed that $i_0 < p$
	\begin{equation*}
	B 
	\leq  L \, \frac{\Delta T_{i_0+1} }{K_{i_0 + 1}^\alpha}\, ,
	\end{equation*}
	i.e., using $ K_{i_0} \, \Delta  T_{i_0}   = 2^{2p+2}$, 
	\begin{equation*}
	B 
	\leq 2^{1 + \alpha} L \,  \frac{\Delta T_{i_0}}{K_{i_0}^\alpha}  
	= 2^{1 + \alpha}  L \, (\Delta T_{i_0})^{1 + \alpha} \, 2^{-(2 p  +2)\alpha} \, .
	\end{equation*}
	From this we deduce, using $2^{p} \geq  B$ for the second inequality,
	\begin{equation*}
	\big(\Delta T_{i_0}\big)^{(1+\alpha)} 
	\geq 2^{-1 - \alpha} B L^{-1} 2^{ (2p+2)\alpha }
	\geq 2^{-1 + \alpha} L^{-1} B^{2\alpha +1} \, .
	\end{equation*}
	Hence, using $2^{(\alpha  - 1) / (\alpha + 1)} \geq 1/2$, we obtain $ \Delta T_{i_0}	\geq (1/2) L^{-1 / (\alpha + 1)} B^{(2 \alpha  + 1)/ (\alpha + 1)}  $, thus
	$
		\ifrac{B}{\Delta T_{i_0}} \leq 2\,L^{1 / (\alpha + 1)} B^{- \alpha /(\alpha +1)} \, .
	$
\end{proof}
\subsection{Anytime-MeDZO and proof}\label{app:anytimeproof}
The doubling trick is the most standard way of converting non-anytime algorithms into anytime algorithms, when the regret bound is polynomial. It consists in taking fresh starts of the algorithm over a grid of dyadic times. The implementation of the trick is straightforward in our case.
	\begin{algorithm}[H]
		\begin{algorithmic}[1]
			\caption{Doubling trick MeDZO}
			\label{alg:adaptivediscretizationdoublingtrick}
			\STATE \textbf{Input}: parameter $m \in [1/2, 1]$;
			\FOR{ $i = 0, \,  \dots $ }
			\STATE Run MeDZO (Alg. \ref{alg:adaptivediscretization}) with input $B = 2^{im}$ for $2^i$ rounds
			\ENDFOR
		\end{algorithmic}
	\end{algorithm}
\begin{corollary*}[Doubling trick version]
	Choose $m \in [1/2, 1]$. The doubling-trick version of MeDZO, run with $m$ as sole input (and without the knowledge of T) ensures that for all regularity parameters $\alpha >0$ and $L > 0$ and for $T \geq 1$
	\begin{equation*}
	\! \!\sup_{f \in \mathcal H (L, \alpha)} \!\! \Rt 
	\leq  4000 (\log_2 T^m )^{3/2}  \max \big( T^m , T L^{1/(\alpha + 1) }(T^m)^{-\alpha / (\alpha + 1)} \big) \!
	= \! {\mathcal O} \big( (\log T)^{3/2}\, T^{\theta_m(\alpha)}\big) .
	\end{equation*}
\end{corollary*}
As the regret bound is not exactly of the form $cT^{\theta}$, we work with the polynomial version of the bound on the regret of MeDZO, equation~\eqref{eq:almostfinalbound}, for the doubling trick to be effective. Obviously the value of the constant in the bound is not our main focus, but we still write it explicitly as it shows that there is no hidden dependence on the various parameters.
\begin{proof}
	By \eqref{eq:almostfinalbound}, with $p_i = \ceil{\log_2 2^{im}} \leq  1 + \log_2 2^{im}$, in the $i$-th doubling trick regime, the cumulative regret is bounded by
	\begin{equation*}
		2(1 + 72 \sqrt{ 1+ \log_2 2^{im} })2^i  L^{1 / (\alpha + 1)} (2^{im})^{-\alpha /(\alpha + 1)} +  (1 + \log_2 2^{im}) \big(1 + 72 \sqrt{1 + \log_2 2^{im} } \big) 2^{im}
	\end{equation*}
	Now since
	\begin{equation*}
	\sum_{i = 0}^{\ceil{\log_2 T}} 2^i= 2^{\ceil{\log_2 T} + 1} - 1 \geq 2 T -1 \geq T \, ,
	\end{equation*}
	there are always less than $\ceil{\log_2 T}$ full regimes.
	Therefore, using $\log_2 2^{im} \leq \log_2 T^m$, and summing over the regimes, the first part of this sum is bounded by
	\begin{equation*}
	\begin{split}
		& 2 ( 1 + 72 \sqrt{2 \log_2 T^m})L^{1 / (\alpha + 1)}  \sum_{i = 0}^{\ceil {\log_2 T}} 2^{i(1 - m \alpha / (\alpha + 1))}\\
		& \leq 2 ( 1 + 72 \sqrt{2 \log_2 T^m})L^{1 / (\alpha + 1)} \frac{2^{(\ceil{\log_2 T}+1) (1 - m \alpha / (\alpha +1)) }}{ 2^{1 - m \alpha / (\alpha + 1)} - 1}  \\
		& \leq 2 (1 + 72\sqrt 2) \sqrt{\log_2 T^m} L^{1 /(\alpha + 1)} \frac{2^{2(1 - m \alpha / (\alpha + 1))}}{\sqrt 2 -1} T (T^{m})^{-\alpha / (\alpha + 1)} \\
		& \leq 2 (1 + 72\sqrt 2) \sqrt{\log_2 T^m} L^{1 /(\alpha + 1)} \frac{4}{\sqrt 2 -1} T (T^{m})^{-\alpha / (\alpha + 1)}
	\end{split}
	\end{equation*}
	where we used $2^{\ceil{\log_2 T}} \leq 2 T$; we also used the fact that since $m\geq 1/2$, we always have the inequality $1 - m\alpha /( \alpha + 1) \geq 1/ 2$ to bound the denominator. Similarly, the second part is bounded by
	\begin{equation*}
		2(1 + 72 \sqrt 2 ) (\log_2 T^m)^{3/2}  \sum_{i = 0}^{\ceil {\log_2 T}} 2^{im}
		\leq  2(1 + 72 \sqrt 2 ) (\log_2 T^m)^{3/2} \frac{4}{\sqrt 2 -1} T^m \, .
	\end{equation*}
	All in all, we obtain the same minimax guarantees as if we had known the time horizon in advance, but with an extra multiplicative factor of $4 / (\sqrt2 -1) \approx 9,66$.
\end{proof}

\subsection{Illustration}\label{app:illustration}
In this section we provide a figure to illustrate the behavior of MeDZO in a schematic example. 

MeDZO starts by playing on a fine discretization with a size of order $\sqrt T$, but for a short length of time, of order $\sqrt T$. At the end of the first epoch, it memorizes the empirical distribution of the arms played; then it runs a new instance of CAB1.1 with both the coarser discretization, and the memorized action. This process is repeated until the time horizon is reached.

The payoffs of the memorized actions increase until the size of the discretization reaches a critical value; after that they fluctuate. Therefore MeDZO manages to maintain a regret of order the approximation error at this critical discretization, multiplied by $T$. 
\begin{figure}[H]
	\center
	{\includegraphics[width=\linewidth]{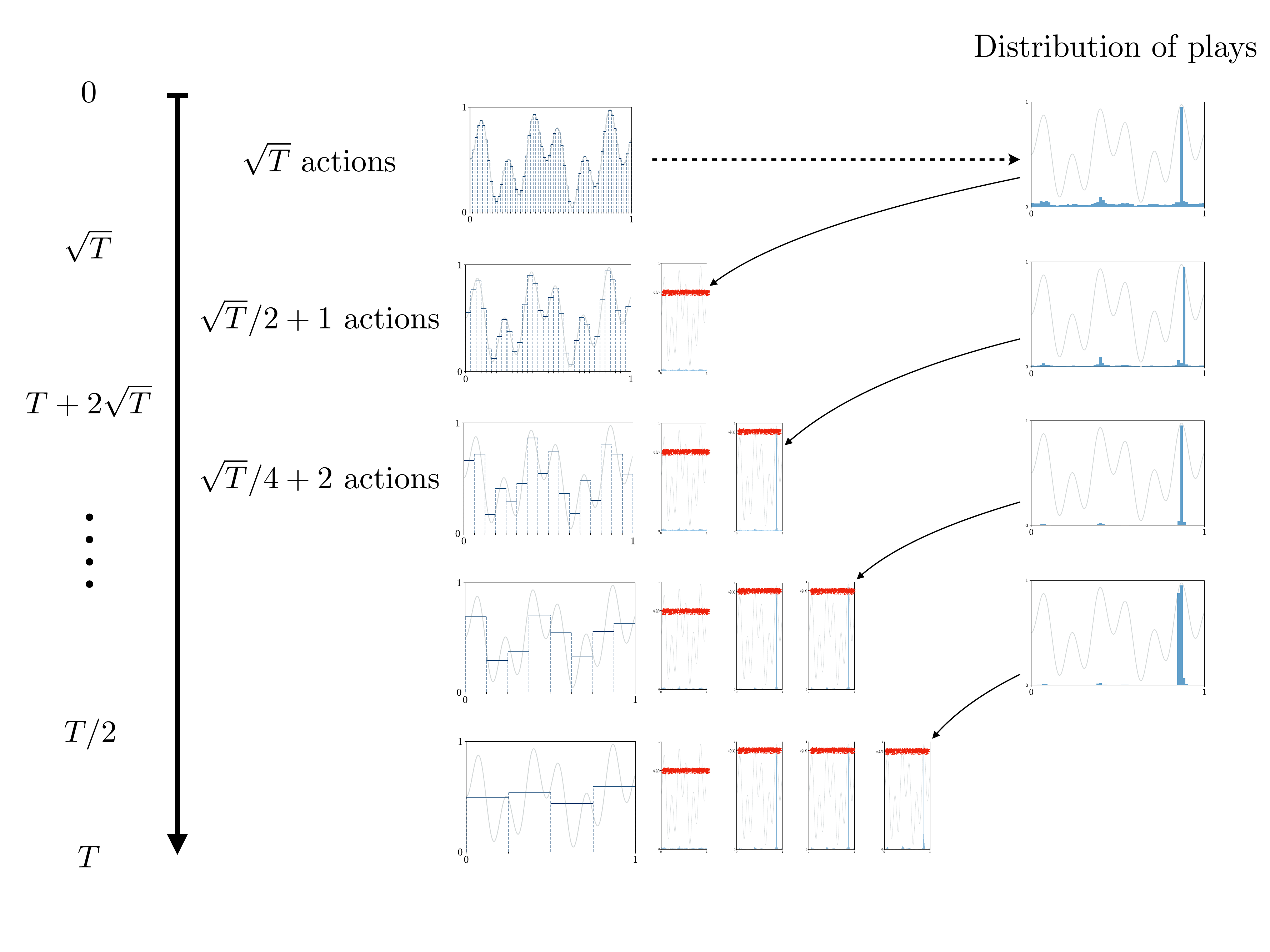}}
	\caption{Behavior of MeDZO on a schematic drawing. The expected payoffs of the memorized actions are displayed in red; those from the usual discretization are in blue. }
	\label{fig:illustration}
\end{figure}

\subsection{Numerical experiments}\label{app:experiments}
This section contains some numerical experiments comparing the regrets of algorithms that require the knowledge of the smoothness, against MeDZO.

We examine bandit problems defined by their mean-payoff functions and gaussian $\mathcal N(0; 1/4)$ noise. The functions considered are $f : x \mapsto (1/2)\sin(13x)\sin(27x) +0.5$ taken from \citet{bubeck2011Xarmed}, $g :x \mapsto  \max\big(3.6\, x(1-x), 1- 1/0.05 \abs{x-0.05}\big)$ adapted from \citet{coquelin2007bandit} and the Garland function $x \mapsto x(1-x) (4-\sqrt{\abs{\sin(60x)}}$, which we took from \citet{valko2013Stochastic}. The functions are plotted in Figure~\ref{fig:problem_plots}.
\begin{figure}[h]
	\centering
	\begin{subfigure}{.33\textwidth}
		\centering
		\includegraphics[width=\linewidth]{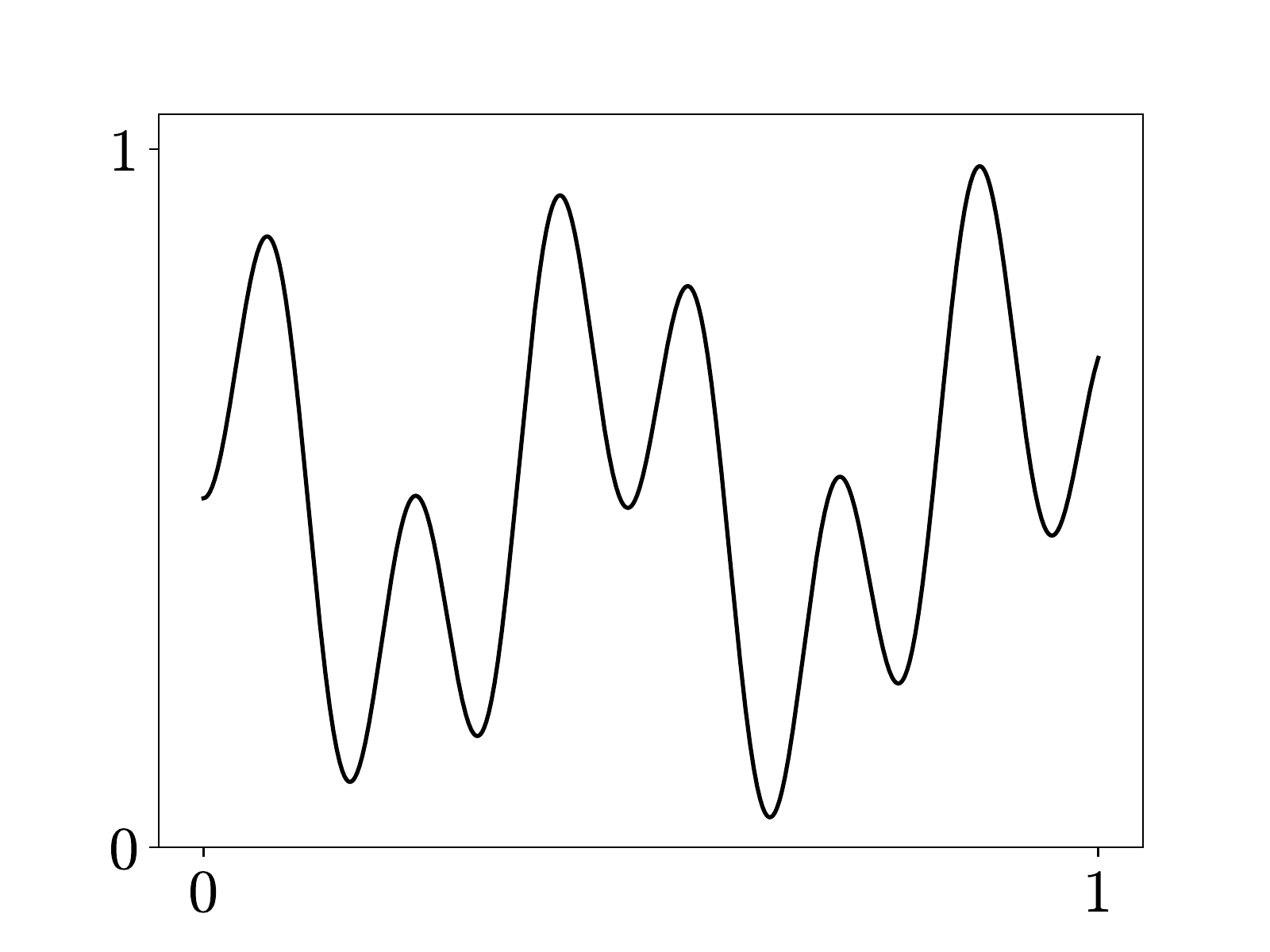}
		\caption{$f $}
		\label{fig:functionf}
	\end{subfigure}%
	\begin{subfigure}{.33\textwidth}
		\centering
		\includegraphics[width=\linewidth]{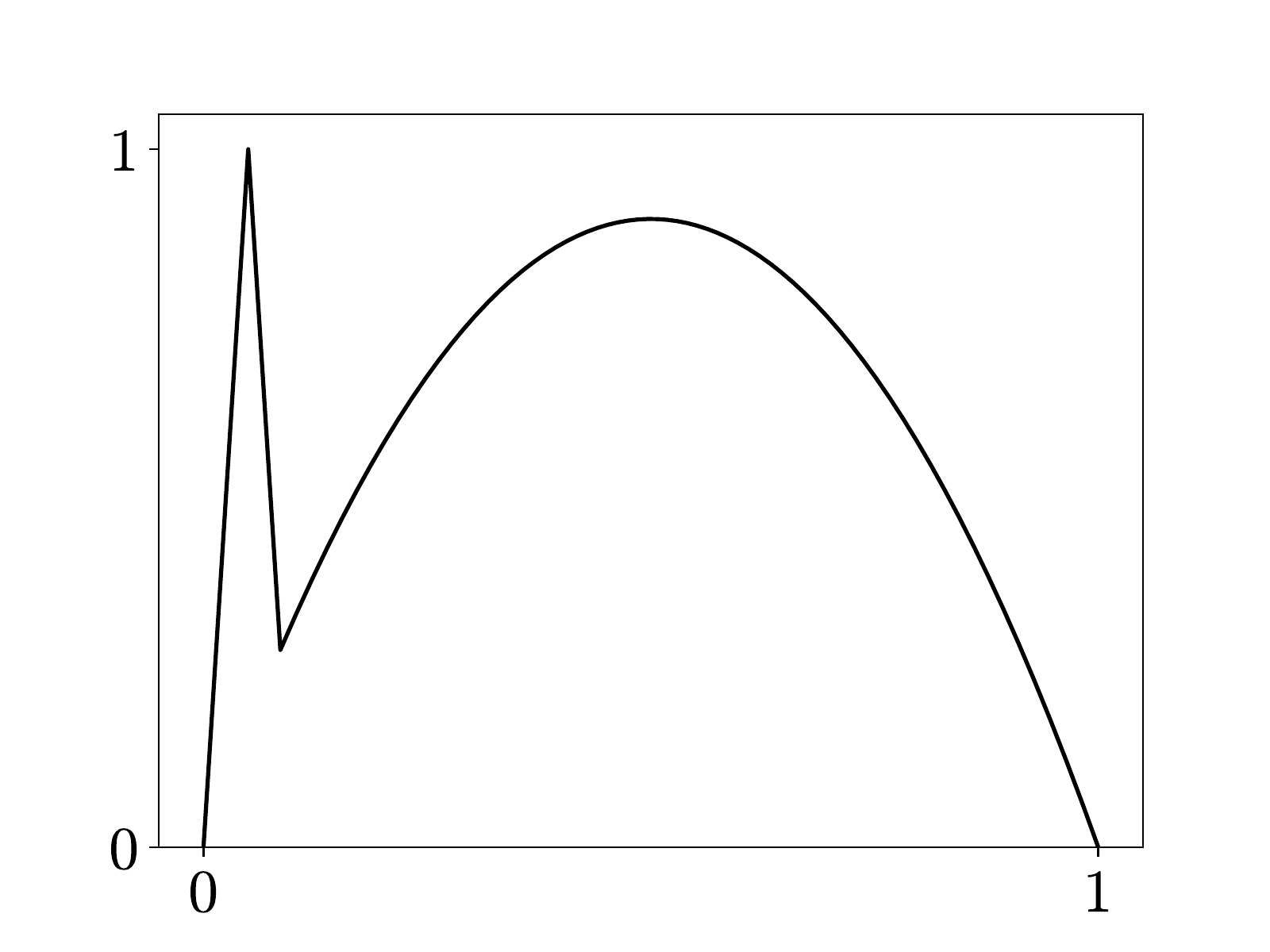}
		\caption{$g$ }
		\label{fig:functiong}
	\end{subfigure}
	\begin{subfigure}{.33\textwidth}
		\centering
		\includegraphics[width=\linewidth]{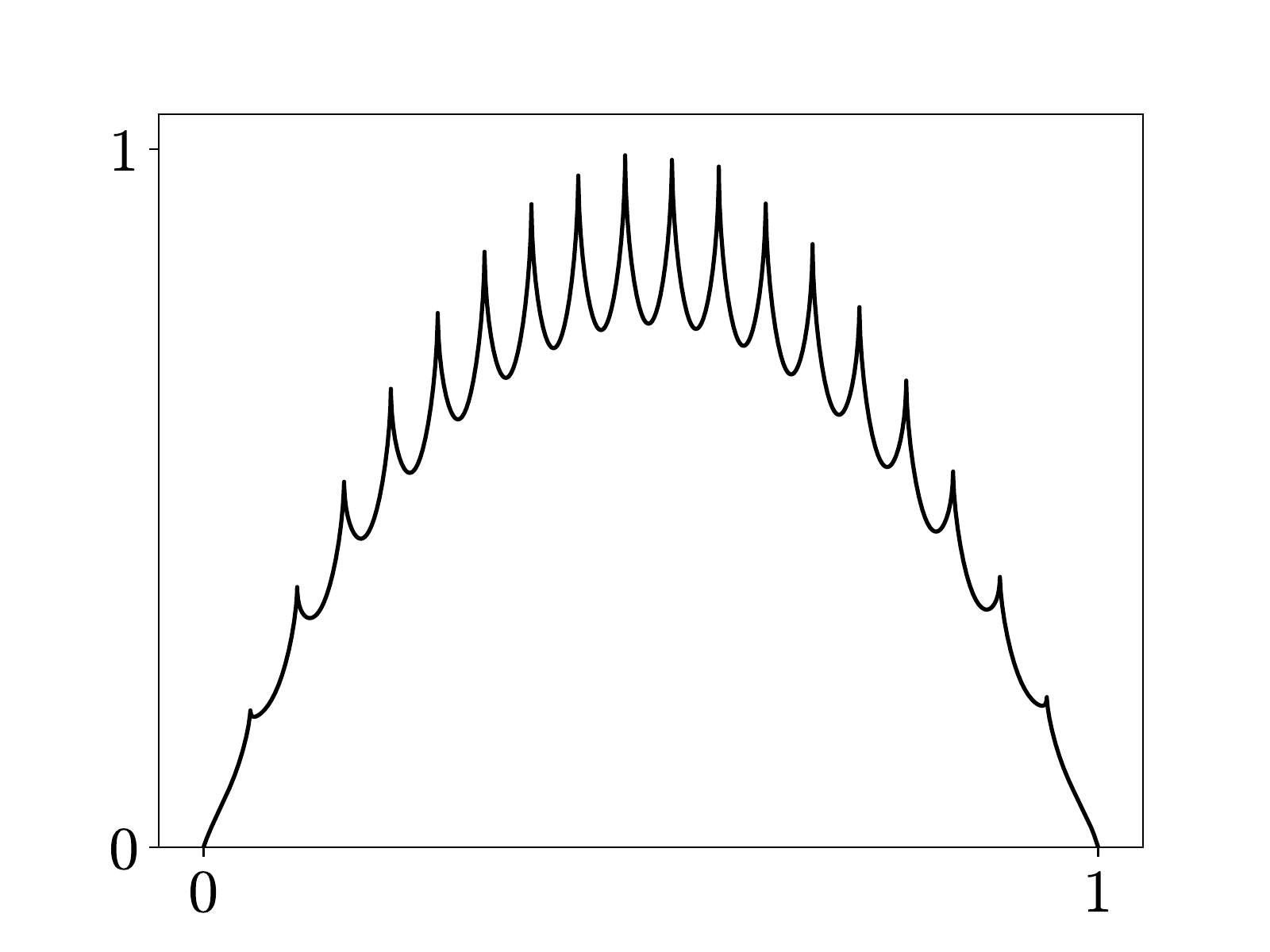}
		\caption{The Garland function}
		\label{fig:functiongar}
	\end{subfigure}
	\caption{Problems considered}
	\label{fig:problem_plots}
\end{figure}

The algorithms we compare are SR from \citet{locatelli2018adaptivity}, and CAB1 from \citet{kleinberg2004Nearly} with MOSS as the discrete algorithm. SR takes directly the smoothness $\alpha$ as an input, and assumes $L=1$. For CAB1, we compute the optimal discretization size for $L=1$ and varying $\alpha$. 

\begin{figure}[h]
	\centering
	\begin{subfigure}{.33\textwidth}
		\centering
		\includegraphics[width=\linewidth]{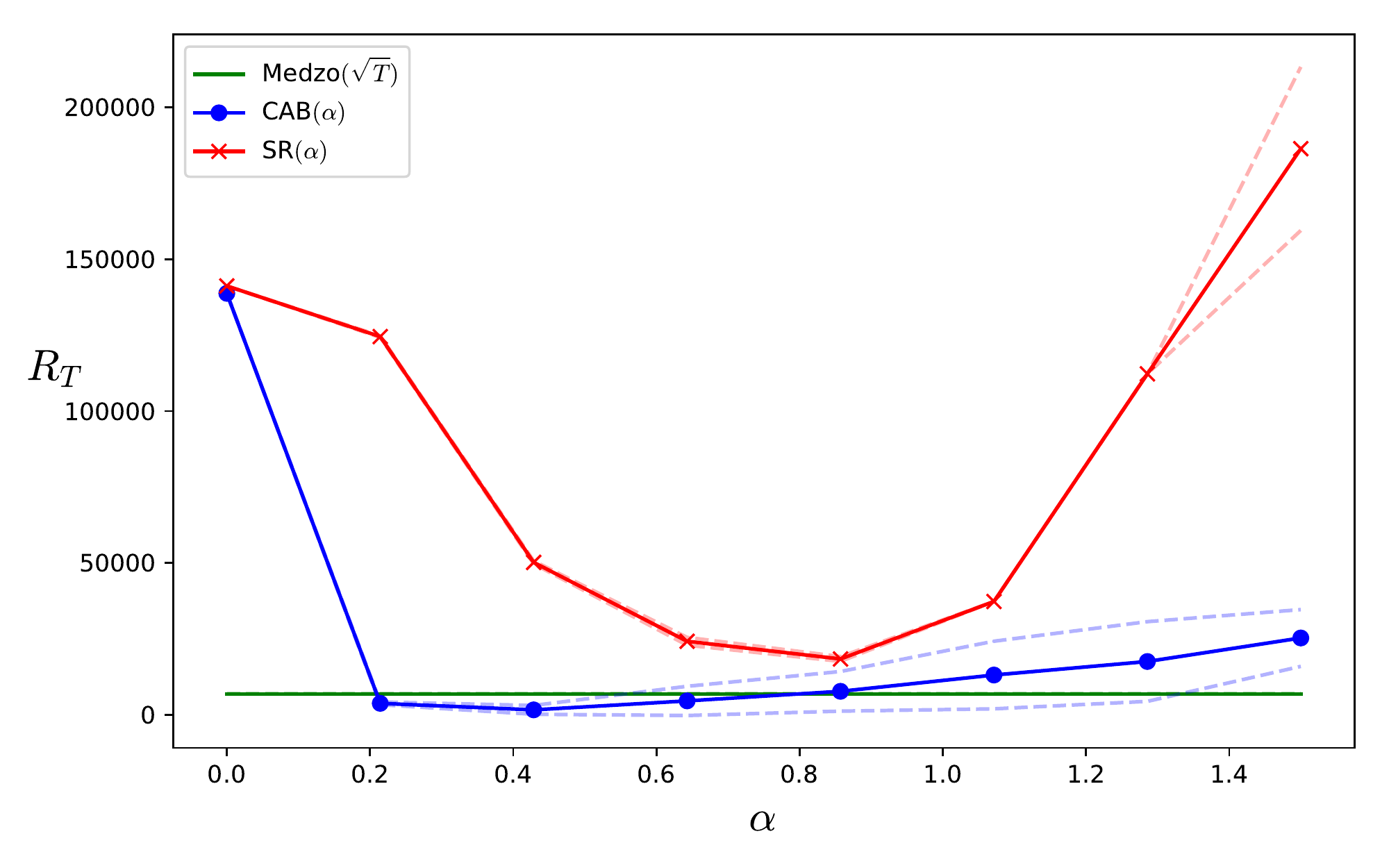}
		\caption{$f $}
		\label{fig:regretf}
	\end{subfigure}%
	\begin{subfigure}{.33\textwidth}
		\centering
		\includegraphics[width=\linewidth]{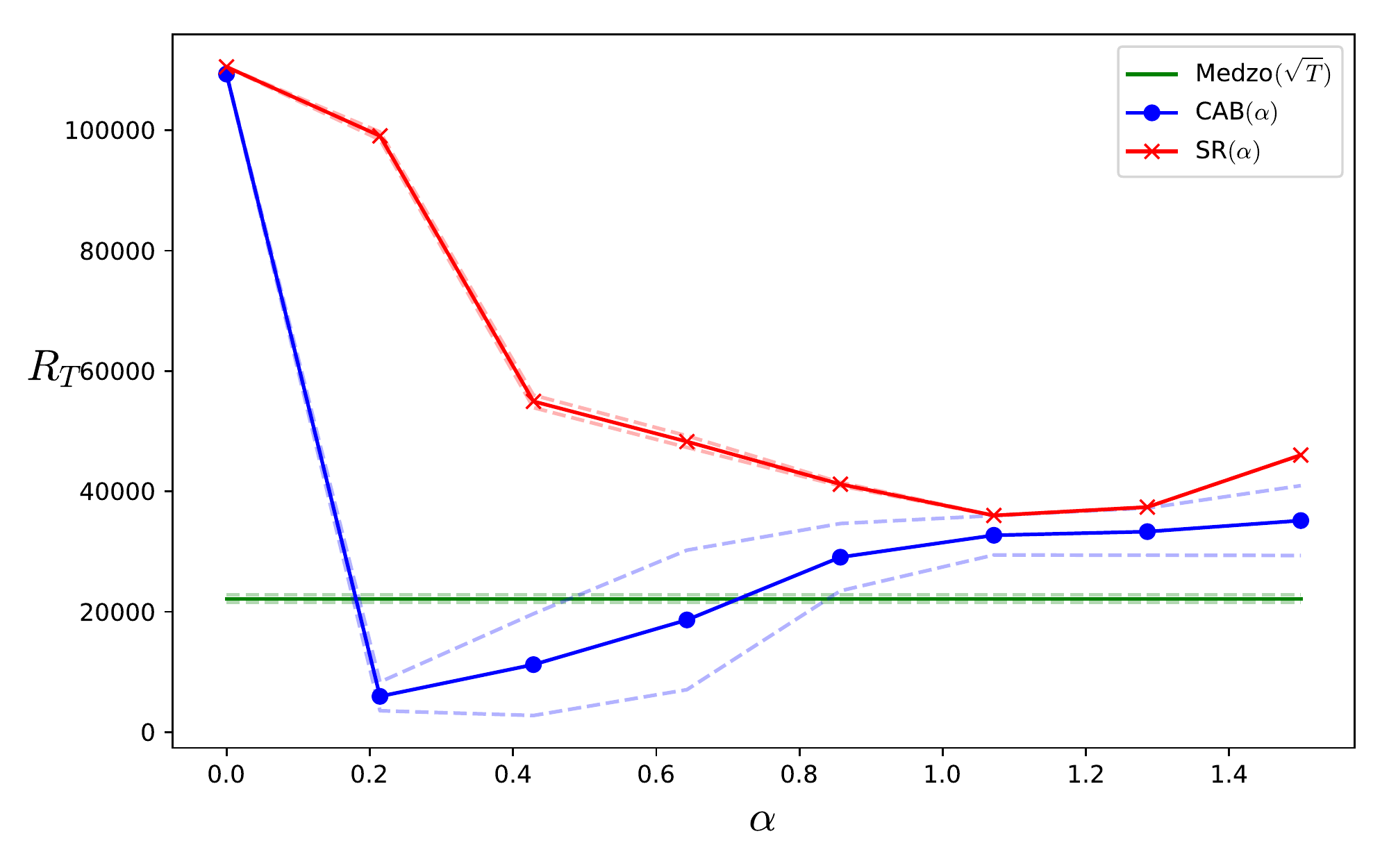}
		\caption{$g$ }
		\label{fig:regretg}
	\end{subfigure}
	\begin{subfigure}{.33\textwidth}
		\centering
		\includegraphics[width=\linewidth]{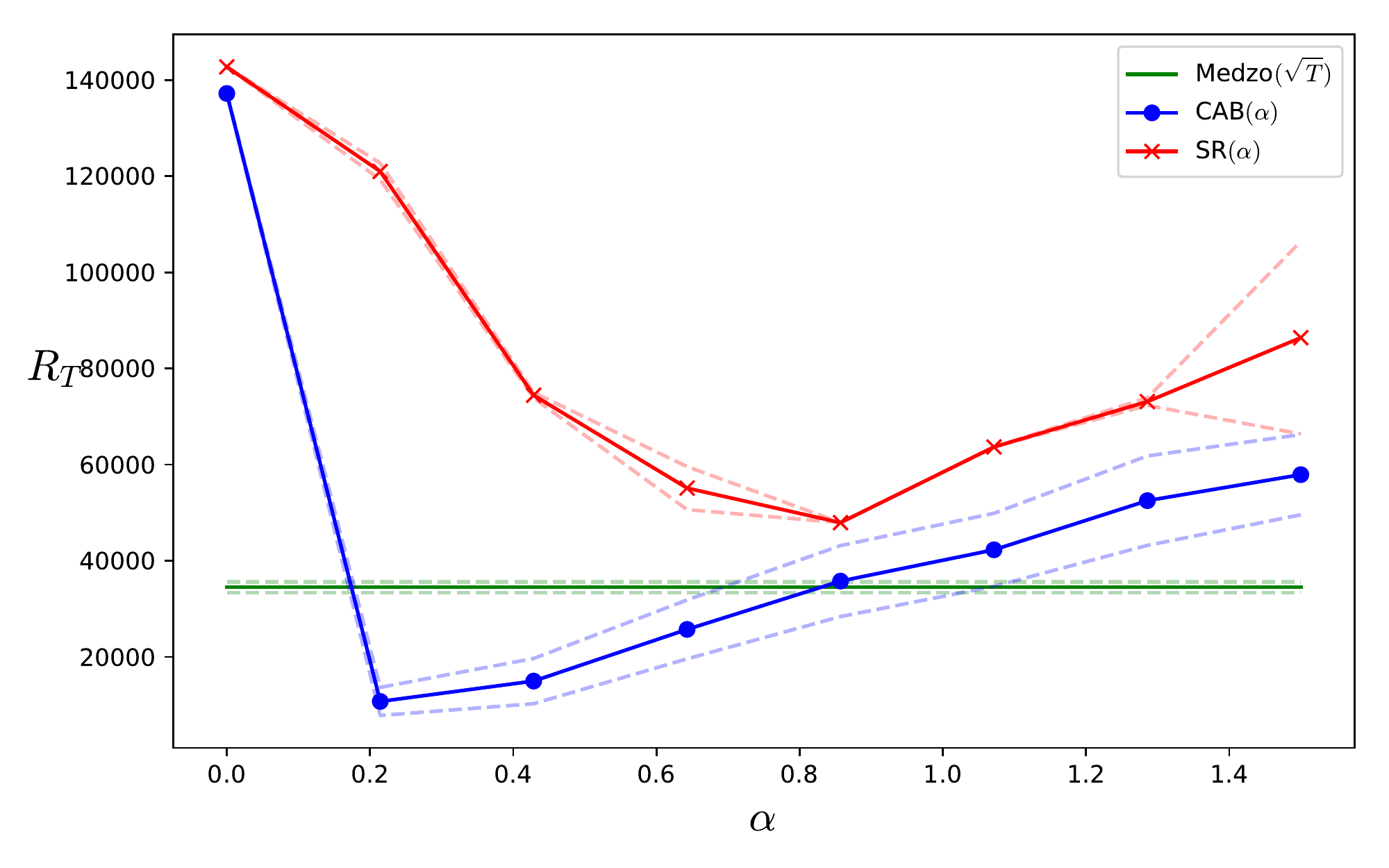}
		\caption{The Garland function}
		\label{fig:regretgar}
	\end{subfigure}
	\caption{Regrets of MeDZO, and of SR and CAB1 run with different values of the smoothness parameter.}
	\label{fig:regrets}
\end{figure}

In Figure \ref{fig:regrets} we plot the cumulative regret of the algorithms after a time horizon $T = 300000$, for varying values of the assumed smoothness. For each problem, MeDZO was run only once, as it does not need to know the smoothness. The regret was averaged over $N=75$ runs, and the dotted curves represent +/- one standard deviation. 

We recall that minimax guarantees are worst-case guarantees, therefore comparing algorithms on a single problem can only serve as an empirical illustration.

As expected, the regrets of both SR and CAB1 depend on some careful tuning of the input parameter, determined by the smoothness. The optimal tuning is unclear, and seems to vary on the algorithm.  MeDZO, on the other hand, obtains reasonable regret with no tuning. Surprisingly, CAB1 with overestimated smoothness seems to behave quite well, although the large variance sometimes makes it difficult to distinguish the results. Recall that MeDZO is the only algorithm with theoretical guarantees for high values of $\alpha$.

\subsection{About simple regret}\label{app:simpleregret}

In this section, we consider the case of \emph{simple regret}, which complements the discussion about adaptation to smoothness in sequential optimization procedures. We write out how to achieve adaptation at usual rates for simple regret under Hölder smoothness assumptions. We do not claim novelty here, as adaptive strategies have already been used for simple regret under more sophisticated regularity conditions (see, e.g.,  \citet{grill2015Blackbox}, \citet{shang2019General} and a sketched out procedure in \citet{locatelli2018adaptivity}); however, we feel the details deserve to be written out in this simpler setting.

Let us recall the definition of simple regret. In some cases, we may only require that the algorithm outputs a recommendation $\tilde X_T$ at the end of the $T$ rounds, with the aim of minimizing the simple regret, defined as
\begin{equation*}
\overline r_T = M(f) - \E \big[ f\big(\tilde X_T\big) \big] \, .
\end{equation*}
This setting is known under various names, e.g., pure exploration, global optimization or black-box optimization. As noted in \citet{bubeck2011Pure}, minimizing the simple regret is easier than minimizing the cumulative regret in the sense that if the decision-maker chooses a recommendation uniformly among the arms played $X_1, \dots, X_T$, then
\begin{equation}\label{eq:simplefromcum}
\overline r_T
= M(f) - \frac{1}{T} \sum_{t = 1}^{T} \E\big[ f\big(X_t\big)\big] 
= \frac{\Rt}{T} \, . 
\end{equation}

The minimax rates of simple regret over Hölder classes $\mathcal H(L, \alpha)$ are lower bounded by $\Omega(L^{1/(2\alpha +1)}T^{-\alpha / (2\alpha +1)})$, which are exactly the  rates for cumulative regret divided by $T$ (see \citet{locatelli2018adaptivity} for a proof of the lower bound). Consequently, at known regularity, any minimax optimal algorithm for cumulative regret automatically yields a minimax recommendation for simple regret via \eqref{eq:simplefromcum}.

When the smoothness is unknown, the situation turns out to be quite different. Adapting to the Hölder parameters can be done at only a (poly-)logarithmic cost for simple regret, contrasting with the polynomial cost of adaptation of cumulative regret. This can be achieved thanks to a very general and simple cross-validation scheme defined in \citet{shang2019General}, named General Parallel Optimization.

\renewcommand{\algorithmiccomment}[1]{#1}
\begin{algorithm}[H]
	\begin{algorithmic}[1]
		\caption{GPO (\underline{G}eneral \underline{P}arallel \underline{O}ptimization) for Hölder minimax adaptation}
		\label{alg:gpoholder}
		\STATE \textbf{Input}: time horizon $T\geq 8$
		\STATE \textbf{Set}: $p = \ceil{\log_2 T}$ and define $K_i = 2^{i}$ for $i = 1, \dots, p$   
		\FOR[\hfill $/ \! /$ Exploration]{$i=1, \ldots, p$ } 
		\STATE For $\floor{T/(2p)}$ rounds, run algorithm \algname{CAB1.1} with the discretization in $K_i$ pieces; use \algname{MOSS} as the discrete algorithm
		\STATE Define output recommendation $\tilde X^{(i)}$, uniformly chosen among the $\floor{ T / (2p)}$ arms played 
		\ENDFOR
		\FOR[\hfill $/ \! /$ Cross-validation]{$i=1, \ldots, p$ } 
		\STATE Play $\floor{T / (2p)}$ times each recommendation $\tilde X^{(i)}$ and compute the  average reward $\hat \mu^{(i)}$
		\ENDFOR

		\RETURN A recommendation $\tilde X_T = \tilde X^{(\hati  )} $ with $\hati   \in \argmax \hat \mu^{(i)}$
	\end{algorithmic}
\end{algorithm}
The next result shows that the player obtains the same simple regret bounds as when the smoothness is known (up to logarithmic factors). 
\begin{theorem}\label{thm:simple}
	\algname{GPO} with \algname{CAB1.1} as a sub-algorithm (Alg. \ref{alg:gpoholder}) achieves, given $T \geq 8$ and without the knowledge of $\alpha$ and $L$, for all $\alpha > 0$ and $L\geq 2^{\alpha + 1/2} \sqrt{\ceil{\log_2 T} /T}$ the bound
	\begin{equation*}
	\sup_{f \in \mathcal H (L, \alpha)} \overline r_T
	\leq \big(54 + \frac{\sqrt{\pi}}{2} \log_2 T \big)         L^{1 / (2\alpha +1)} \bigg(\frac{\ceil{\log_2 T}}{T}\bigg)^{\alpha / (2\alpha + 1)} \! \!  = \tilde {\mathcal O}\! \left(L^{1 / (2\alpha + 1)} T^{-\alpha/ (2\alpha + 1)} \right) \, . 
	\end{equation*}
\end{theorem}
The $\tilde {\mathcal O }$ notation hides the $\log T$ factors, and the assumption that $T \geq 8$ is needed to ensure that $ T / (2 p)  = T / (2 \ceil{\log_2 T})\geq 1$: otherwise the algorithm itself is ill-defined. 

\begin{proof}
	Let $f\in \mathcal H(L, \alpha)$ denote a mean-payoff function. Once again we decompose the error of the algorithm into two sources. The simple regret is the sum of the regret of the best recommendation among the $p$ received, $\rmin$, and of a cross-validation error, $\rcv$,
	\begin{equation}\label{eq:simpledecomp}
	M(f) - \E[f(\tilde X _T)]
	=  \underbrace{\min_{i = 1, \dots, p} \bigg( M(f) - \E\Big[f\big(\tilde X^{(i)} \big)\Big]  \bigg)}_{\rmin} + \underbrace{\max_{i = 1, \dots, p} \bigg( \E\Big[f\big(\tilde X^{(i)} \big)\Big]  - \E\! \left[ f\big(\tilde X_T\big) \right] \bigg)}_{\rcv} \, .
	\end{equation}
	We now show that
	$
	\rcv \leq p^{3/2}\sqrt {\pi /(4T)} \, ,
	$ by detailing an argument that is sketched in the proof of Thm. 3 in \citet{shang2019General}.
	Denote by $\hat \mu^{(i)}$ the empirical reward associated to recommendation $i$, and $\hati  = \argmax \hat \mu^{(i)}$, so that $\tilde X_T = \tilde X^{(\hati)}$. Then for any fixed $i$, by the tower rule,
	\begin{equation}\label{eq:towerrulemui}
		\E\big[ \hat{\mu}^{(i)}\big] 
		= \E \! \left[  \E\! \left[\hat{\mu}^{(i)} \,  \Big\vert \,  \tilde X^{(i)} \right] \right] 
		= \E \! \left[f \big(\tilde X^{(i)}\big) \right] \, . 
	\end{equation}
	Therefore, by the above remarks, and since $\hat \mu^{(i)} \leq \hat \mu^{(\hati)}$,
	\begin{equation*}
		\E\! \left[ f\big(\tilde{X}^{(i)}\big) \right] - \E\!\left[ f\big(\tilde X_T\big)\right]
		=  \E\! \left[\hat{\mu}^{(i)} - f\big(\tilde X^{(\hati)}\big)\right]
		\leq \E\! \left[\hat{\mu}^{(\hati)} - f\big(\tilde X^{(\hati)}\big)\right] \, .
	\end{equation*}
	We have to be careful here, as $\hati$ is a random index that depends on the random variables $\hat \mu^{(i)}$'s:  we cannot apply directly the tower rule as in \eqref{eq:towerrulemui}. To deal with this, let us use an integrated union bound. Denote by $( \, \cdot\, )^+$ the positive part function, then
	\begin{equation*}
		\E\! \left[\hat{\mu}^{(\hati)} - f\big(\tilde X^{(\hati)}\big)\right]
		\leq \E\! \left[ \left( \hat{\mu}^{(\hati)} - f\big(\tilde X^{(\hati)}\big) \right)^{+} \right]
		\leq \sum_{j = 1}^{p}\E\! \left[ \left(\hat{\mu}^{(j)} - f\big(\tilde X^{(j)}\big) \right)^+ \right] \, ,
	\end{equation*}
	and we are back to handling empirical means of i.i.d. random variables. For each $j$, the reward given $\tilde X^{(j)}$ is $(1/4)$-subgaussian. Therefore, as $\hat \mu^{(i)}$ is the empirical mean of $n = \floor{T / (2p)}$ plays of the same arm $\tilde X^{(j)}$, this mean $\hat \mu^{(i)}$ is $(1 /(4n))$-subgaussian conditionally on $\tilde X^{(j)}$ and thus for all $\epsilon	> 0$
	\begin{equation*}
		\P \! \left[ \hat{\mu}^{(j)} - f\big(\tilde X^{(j)}\big) \geq \epsilon \right] 
		\leq e^{-2 n \epsilon^2} \, . 
	\end{equation*}
	Hence by integrating over $\epsilon \in [0, + \infty)$, using Fubini's theorem, a change of variable $x = \sqrt{4n} \epsilon$ (and using the fact that $\floor{T / (2p)} \geq T / (4p) $ as $T / (2p) \geq 1$):
	\begin{equation*}
	\begin{split}
		\E\! \left[ \left(\hat{\mu}^{(j)} - f\big(\tilde X^{(j)}\big) \right)^+ \right]
		&= \int_{0}^{+\infty }\! \! \P \! \left[ \hat{\mu}^{(j)} - f\big(\tilde X^{(j)}\big) \geq \epsilon \right] \! \!   \dint \epsilon \\
		&\leq \int_{0}^{+\infty } e^{-2n \epsilon^2} \!  \dint \epsilon 
		= \frac{1}{\sqrt{4n}}  \int_{0}^{+\infty} e^{- x^2/2} \dint x\\
		&= \sqrt{\frac{\pi}{8n}} 
		= \sqrt{\frac{\pi}{8 \floor{T / 2p}}} 
		\leq \sqrt{\frac{ \pi p}{4 \, T}}
	\end{split}
	\end{equation*}
	Putting back the pieces together, we have shown that for any $i$,
	\begin{equation*}
			\E\! \left[ f\big(\tilde{X}^{(i)}\big) \right] - \E\!\left[ f\big(\tilde X_T\big)\right]
			\leq \sum_{j =1}^{p} \sqrt{\frac{\pi p}{4 \, T}}
			= p^{3/2} \sqrt {\frac{\pi}{4\, T}}\, .
	\end{equation*}
	We deduce the same bound for $\rcv$ by taking the maximum over $i$.
	
	Let us now  bound $\rmin$. By Eq.~\eqref{eq:decompositionregret}, using the fact that $\floor{T/(2p)} \geq T/(4p)$ as $T / (2p) \geq 1$, for all $i$
	\begin{equation*}
	M(f) - \E\Big[f\big(\tilde X^{(i)} \big)\Big]  
	\leq \frac{L}{K_i^\alpha} + 18 \sqrt{\frac{4pK_i}{T}} \, .
	\end{equation*}
	We summarize a few calculations in the next lemma. These calculations come from the minimization over the $K_i$'s of the previous bound, with a case disjunction arising from the boundary cases. 
	\begin{lemma}\label{lem:minimizingsimplereg}
		At least one of the three following inequalities holds :
		\begin{equation*}
		L < 2^{\alpha + 1/2}\sqrt{\frac{p}{T}}
		\quad 
		\text{or} 
		\quad
		L \geq T^\alpha \, \sqrt{p}
		\end{equation*}
		or
		\begin{equation*}
		\min_{i = 1, \dots, p} \Bigg( \frac{L}{K_i^\alpha} + 36 \sqrt{\frac{pK_i}{T}} \Bigg) 
		\leq 53L^{1 / (2\alpha + 1)} \left(\frac{p}{T} \right)^{\alpha / (2\alpha + 1)}\, .
		\end{equation*}
	\end{lemma}
	Let us consider these three cases separately. The first one is forbidden by the assumption that $L \geq 2^{\alpha  +1 / 2}\sqrt {p/T}$. In the second case, the function is so irregular that  the claimed bound becomes worse than $\overline r_T \leq 56 \, p^{1/2 + \alpha/ (2\alpha + 1)}$, which is weaker than the trivial bound $\overline r_T \leq 1$.
	
	Finally, in the third case, we may assume that $L \geq 2^{\alpha + 1/2}\sqrt{p / T} \geq \sqrt{p / T}$. Then we have 
	\begin{equation*}
	L^{1 / (2\alpha + 1)}
	\geq \left(\frac{p}{T}\right)^{1 / (2(2\alpha + 1))}
	=\left( \frac{p}{T}\right)^{1/2} \left(\frac{p}{T}\right)^{-\alpha/(2\alpha + 1)}\, ,
	\end{equation*}
	and thus $\sqrt{p/T} \leq L^{1 / (2\alpha + 1)} (p/T)^{\alpha/(2\alpha + 1)}$. By injecting the bound of Lemma~\ref{lem:minimizingsimplereg} and the bound on $\rcv$ into \eqref{eq:simpledecomp}:
	\begin{equation*}
	\overline r_T 
	\leq 53 L^{1 / (2\alpha + 1)} \left(\frac{p}{T} \right)^{\alpha / (2\alpha + 1)} + p \sqrt{\frac{\pi}{4}} \sqrt{\frac{p}{T}} 
	\leq  (53 +p  \sqrt{\pi/4})  L^{1 / (2\alpha + 1)} \left(\frac{p}{T} \right)^{\alpha / (2\alpha + 1)}
	\end{equation*}
	and the stated bound holds, since 
	$53 + p \sqrt{\pi / 4} 
	\leq 53 + (\log_2 T + 1 )\sqrt{ \pi /4} 
	\leq 54 + \sqrt{\pi/4} \log_2 T$.
\end{proof}
\begin{proof}[Proof of Lemma \ref{lem:minimizingsimplereg}]
	We upper bound the minimum by comparing the two quantities
	\begin{equation*}
		\frac{L}{K_i^\alpha} \quad \text{v.s.} \quad \sqrt{\frac{p K_i}{T}} \, .
	\end{equation*}
	As the first term is decreasing with $i$, and the second term is increasing with $i$, two extreme cases have te be dealt with. If the first term is always smaller than the second, i.e., even for $i = 1$, then:
	\begin{equation*}
	\frac{L}{2^\alpha} < \sqrt{\frac{p\,  2}{T}} \, .
	\end{equation*}
	This is the first case in the statement of the lemma. Otherwise, the first term might always be greater than the second one, i.e., even for $i =p$ and
	\begin{equation*}
	\frac{L}{2^{\alpha  p}} \geq \sqrt{\frac{p 2^p}{T}}
	\end{equation*}
	which is equivalent to
	\begin{equation*}
	L^2 \geq p \frac{2^{p (2\alpha + 1)}}{T}\, ,
	\end{equation*}
	hence, since $2^p \geq T$,
	\begin{equation*}
	L^2 \geq p T^{2\alpha}
	\end{equation*}
	which is exactly the second inequality of our statement.

	Otherwise, define $i^\star$ to be an index such that
	\begin{equation}\label{eq:istardef}
		\frac{L}{K_{i^\star - 1}^\alpha} \geq \sqrt{\frac{p K_{i^\star - 1}}{T}} 
		 \quad \text{and} \quad 
		 \frac{L}{K_{i^\star}^\alpha} \leq \sqrt{\frac{p K_{i^\star}}{T}}
	\end{equation}
	By the preceding discussion, $i^\star$ is well defined and $1 < i^\star \leq p$. Then by definition of $i^\star$ (the first equation in \eqref{eq:istardef})
	\begin{equation*}
		2^{\alpha  + 1/2}\frac{L}{K_{i^\star}^\alpha} \geq \sqrt{\frac{pK_{i^\star}}{T}} \, .
	\end{equation*}
	Hence, by squaring and regrouping the terms
	\begin{equation*}
		K_{i^\star}^{2\alpha + 1} \leq 2^{2 \alpha + 1} L^2 \frac{T}{p}
	\end{equation*}
	thus
	\begin{equation*}
		K_{i^\star} \leq 2 L^{2 / (\alpha + 1)} \left( \frac{T}{p}\right)^{1 /(2\alpha + 1)}
	\end{equation*}
	and 
	\begin{equation*}
		\sqrt{\frac{p K_{i^\star}}{T}} \leq \sqrt 2 L^{1/(2\alpha + 1)} \left(\frac{p}{T}\right)^{\alpha / (2\alpha + 1)}
	\end{equation*}
	and finally, recalling the second equation in \eqref{eq:istardef}
	\begin{equation*}
	\frac{L}{K_{i^\star}^\alpha} + 36 \sqrt{\frac{pK_{i^\star}}{T}}  
	\leq 37 \sqrt{\frac{p K_{i^\star}}{T}} 
	\leq 37 \sqrt 2 L^{1 / (2\alpha + 1)} \bigg( \frac{p}{T}\bigg)^{\alpha / (2\alpha + 1)}\, .
	\end{equation*}
\end{proof}

\subsection{Proof of our version of the lower bound of adaptation}\label{app:lowbounds}
Here we provide the full proof of our version of the lower bound of adaptation stated in Section \ref{subsec:lowerbounds}.

 Our statement differs from that of \citet{locatelli2018adaptivity} on some aspects. First, and most importantly, we include the dependence on the Lipschitz constants, and we do not consider margin regularity. We also remove a superfluous requirement on $B$, that $B \leq c \, T^{(\alpha + 1) / (2\alpha + 1)}$, which was just an artifact of the original proof. Furthermore we believe that the additional condition that $L \leq \mathcal O (T^{\alpha / 2})$ in our version was implicitely used in this original proof. Finally, the value of the constant differs,  partly because of the analysis, and partly because we consider $(1/4)$-subgaussian noise instead of $1$-subgaussian noise.

We managed to obtain these improvements thanks to a different proof technique. In the original proof, the authors compare the empirical likelihoods of different outcomes and use the Bretagnolle-Huber inequality. We choose to build the lower bound in a slightly different way (see \citet{garivier2018Explore}): we handle the changes of measure implicitly thanks to Pinsker's inequality (Lemma \ref{lem:dataproc}). Following \citet{lattimore2019book}, we also chose to be very precise in the definition of the bandit model, in order to make rigorous a few arguments that are often used implicitly in the literature on continuous bandits.

 The main argument of the proof, that is, the sets of functions considered, are already present in \citet{locatelli2018adaptivity}.

Before we start with the proof, let us state a technical tool. Denote by $\KL$ the Kullback-Leibler divergence. The next lemma is a generalized version of Pinsker's inequality, tailored to our needs.
\begin{lemma}\label{lem:dataproc}
	Let $\P$ and $\mathbb Q$ be two probability measures. For any random variable $Z \in [0, 1]$, 
	\begin{equation*}
	\abs{\E_{\mathbb P}[Z] - \E_{\mathbb Q} [Z] }  \leq \sqrt{\frac{\KL(\P, \mathbb Q)}{2}}
	\end{equation*}
\end{lemma}
\begin{proof}
	For $z \in [0, 1]$, by the classical version of Pinsker's inequality applied to the event $\{ Z \geq z\}$:
	\begin{equation*}
		\abs{\P[Z \geq z] - \mathbb{Q}[Z \geq z]} \leq \sqrt{\frac{\KL(\P, \mathbb Q)}{2}} \, . 
	\end{equation*}
	Therefore, by Fubini's theorem and the triangle inequality, and by integrating the preceding inequality:
	\begin{equation*}
		\abs{\E_{\mathbb P}[Z] - \E_{\mathbb Q} [Z] } 
		\! = \!  \abs{\int_{0}^{1} \!\!   \big(\P[Z \geq z] - \mathbb{Q}[Z \geq z]\big) \! \dint z} 
		\leq \! \int_{0}^{1} \!\! \abs{\P[Z \geq z] - \mathbb{Q}[Z \geq z]} \! \! \dint z \leq \sqrt{\frac{\KL(\P, \mathbb Q)}{2}}
	\end{equation*}
\end{proof}

\begin{proof}[Proof of the lower bound]
	For the sake of completeness, we recall in detail the construction of \citet{locatelli2018adaptivity}, with some minor simplifications that fit our setting. Fix regularity parameters $\ell, L, \alpha$ and $\gamma$ satisfying $\ell \leq L$ and $\gamma \geq \alpha$, so that $\mathcal 
	H(\ell, \gamma) \subset \mathcal H (L, \alpha)$ (remember the functions are defined on $\mathcal X = [0, 1]$).
	
	Fix $M \in [1/2, 1]$. Let $K \in \mathbb N\setminus\{0\}$ and $\Delta \in \R_{+}$ be some parameters of the construction whose values will be determined by the analysis. We define furthermore a partition of $[0,1]$ into $K+1$ sets, $H_0 = [1/2, 1]$ and $H_i  = [(i-1)/(2K), i / (2K)]$ for $1 \leq i \leq K$, along with their middle points $x_i \in H_i$. Finally, define the set of hypotheses $\phi_i$ for $i =0, \dots, K$ as follows
	\begin{equation}
	\phi_i(x) =
	\left\{
		\begin{split}
		&\max \big(M- \Delta, \, M- \Delta/2 - \ell \abs{x - x_0}^\gamma \big) &&\text{if} \; x \in  H_0 \, ,\\
		& \max \big(M- \Delta, \, M - L \abs{x- x_i}^\alpha \big) && \text{if} \; x \in H_i  \; \text{and} \; s \neq 0\, ,\\
		& M- \Delta \quad &&\text{otherwise.}
		\end{split}\right.
	\end{equation}

	\begin{figure}[h]
		\center
		{\includegraphics[width=10cm]{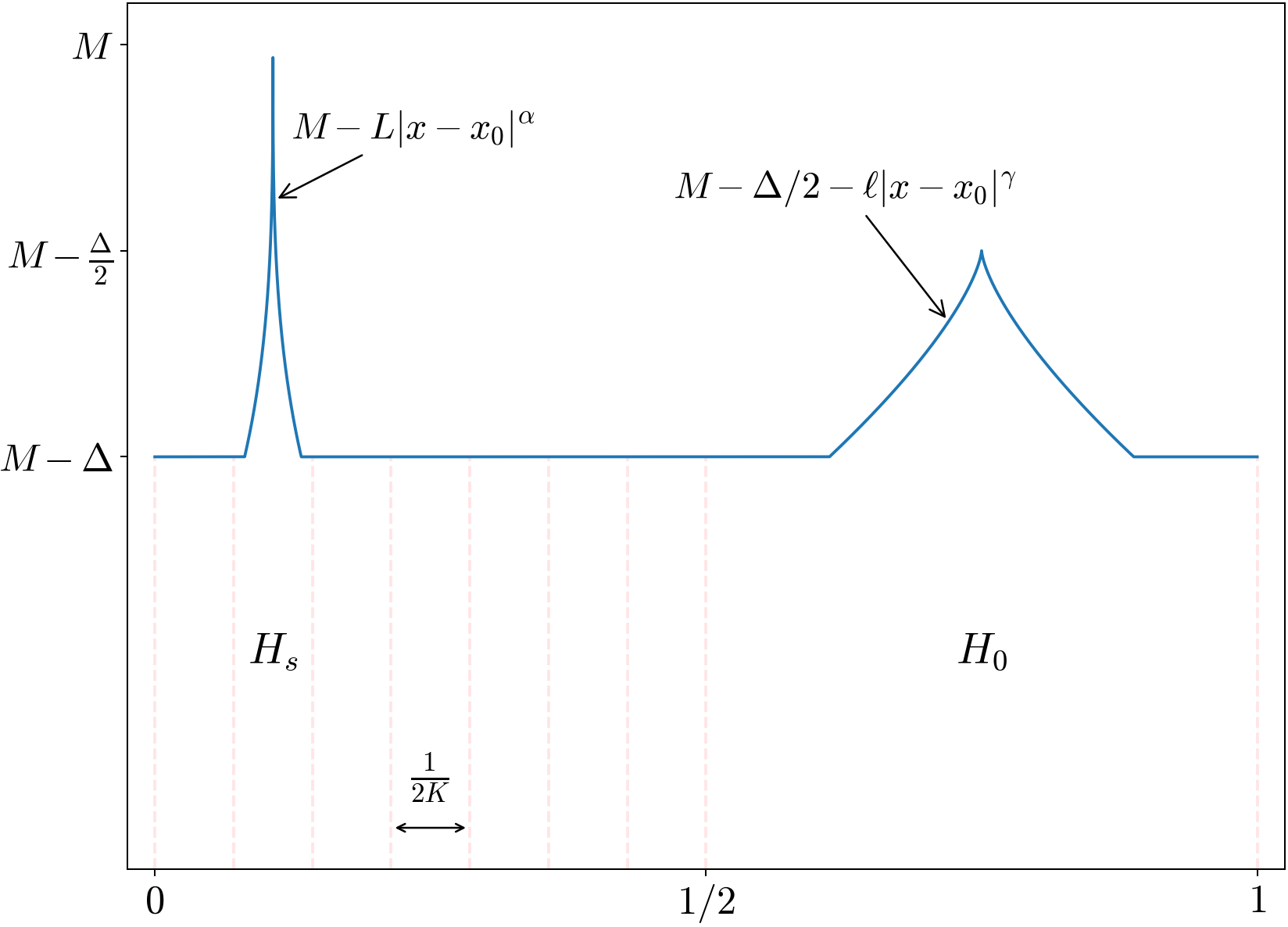}}
		\caption{Mean-payoff functions for the lower bound
		}\label{fig:hypplots}
	\end{figure}
	Figure~\ref{fig:hypplots} illustrates how the $\phi_i$'s are defined : for $1 \leq i \leq K $, the function $\phi_i$ displays a peak of size $\Delta$ and of low regularity $(L, \alpha)$, localized in $H_i$, and another peak of size $\Delta / 2$, of higher regularity $(\ell, \gamma)$ in $H_0$. The function $\phi_0$ only has the peak of size $\Delta / 2$ and regularity $(\ell, \gamma)$. We need to add requirements on the values of the parameters, to make sure the indeed functions belong to the appropriate regularity classes. These requirements are written in the following lemma, which we prove later.
	 \begin{lemma}\label{lem:globreg}
		If $(\Delta / L)^{1/\alpha} \leq 1 / (4K)$  then $\phi_0 \in \mathcal H(\ell, \gamma)$, and if $ (\Delta / (2\ell) )^{1/\gamma} \leq 1 / 4$  then $\phi_i \in \mathcal H(L, \alpha)$ for $i \geq 1$.
	\end{lemma}
	Fix a given algorithm. The idea of the proof of the lower bound is to use the fact that if the player has low regret, that is, less than $B$, when the mean-payoff function is $\phi_0 \in \mathcal H (L, \alpha)$, then she has to play in $H_0$ often. This in turn constrains the amount of exploration she can afford, and limits her ability to find the maximum when the mean-payoff functions is $\phi_i$ for $i > 0$.
	
	\paragraph{Canonical bandit model} In this paragraph, we build the necessary setting for a rigorous development. The continous action space gives rise to measurability issues, and one should be particularly careful when handling changes of measure as we do here. Following \citet[Chap. 4.7, 14 (Ex.11) and 15 (Ex.8) ]{lattimore2019book}, we build the canonical bandit model in order to apply the chain rule for Kullback-Leibler divergences rigorously. To our knowledge, this is seldom done carefully, the two notable exceptions being the above reference and \citet{garivier2018Explore}. We also use the notion of probability kernels in this paragraph; see \citet[Chap.~1 and 5]{kallenberg2006foundations} for a definition and properties.
	
	Define a sequence of measurable spaces $\Omega_{t} = \prod_{s = 1}^{t} \mathcal X \times \R$, together with their Borel $\sigma$-algebra (with the usual topology on $\mathcal X = [0, 1]$ and on $\R$). We call $h_t = (x_1, y_1, \dots, x_t, y_t) \in \Omega_t$ a history up to time $t$. By an abuse of notation, we consider that $\Omega_t \subset \Omega_{t'}$ when $t\leq t'$.
	
	An algorithm is a sequence $(K_t)_{1 \leq t \leq T}$ of (regular) probability kernels, with $K_t$ from $\Omega_{t-1}$ to $\mathcal X$, modelling the choice of the arm at time $t$. By an abuse of notation, the first kernel $K_1$ is an arbitrary measure on $\mathcal X$, the law of the first arm picked. Define for each $i$ another probability kernel modelling the reward obtained: $L_{i , t}$ from $\Omega_t \times \mathcal X$ to $\R$. We write it explicitly as :
	\begin{equation*}
		L_{i, t}\big( (x_1, y_1, \dots, x_t) , B \big) = \sqrt{\frac{2}{\pi }} \int_{B} e^{-2\big(x - \phi_i(x_t)\big)^2} \dint x
	\end{equation*}

	These kernels define probability laws $\P_{i,t} = L_{i, t} (K_{t} \P_{i, t-1}) $ over $\Omega_t$. Doing so, we ensured that under $\P_{i, t}$ the coordinate random variables $X_t : \Omega_t \to \mathcal X$ and $Y_t : \Omega_t \to \R$), defined as $X_t(x_1, \dots, x_t, y_t) = x_t$ and $Y_t(x_1, \dots, x_t, y_t) = y_t$ are such that given $X_t$, the reward $Y_t$ is distributed according to $\mathcal N\big(\phi_i(X_t), 1/4\big)$. Denote by $\E_i$ the expectation taken according to $\P_{i, t}$. We also index recall the pseudo-regret: $\overline R_{T, i} = T M(\phi_i)  - \E_i \! \left[\sum_{t = 1}^{T}\phi_i(X_t) \right]$.
	
	A rewriting of the chain rule for Kullback-Leibler divergence with our notation would be (see \citet[Exercise~11 Chap. 14]{lattimore2019book} for a proof)
	\begin{proposition*}[Chain rule]
		Let $\Omega$ and $\Omega'$ be measurable subsets of $\R^d$ equipped with their natural $\sigma$-algebra. Let $\P$ and $\Q$ be probability distributions defined over $\Omega$, and $K$ and $L$ be regular probability kernels from $\Omega$ to $\Omega' $ then
		\begin{equation*}
			\KL\big( K \P, L \Q \big) = \KL(\P, \Q) + \int_{\Omega} \KL \! \big(K(\omega, \cdot\,), L(\omega, \cdot\,)\big) \dint \P(\omega)
		\end{equation*}
	\end{proposition*}
	The key assumptions are that $\Omega$ and $\Omega'$ are subspaces of $\R^d$, and that $K$ and $L$ satisfy measurability conditions, as they are regular kernels; these assumptions justify the heavy setting we introduced.
	
	Under this setting, we may call to the chain rule twice to see that for any $t$:
	\begin{equation*}
		\begin{split}
			\lefteqn{\KL\big( \P^t_0, \P^t_i \big)
			= \KL\big( L_{0, t} (K_{t} \P^{t-1}_0), L_{i, t} (K_{t} \P^{t-1}_i) \big)}  \\
			&= \KL\big(K_{t} \P^{t-1}_0, K_{t} \P^{t-1}_i \big) + \int_{\Omega_{t-1} \times \mathcal X} \KL\big(  L_{0, t}(h_{t-1}, x_t, \cdot\, ), L_{i, t}(h_{t-1}, x_t, \cdot \,)\big)  \dint K_t\P^{t-1}_0 (h_{t-1}, x_t) \\
			&=\KL\big( \P^{t-1}_0, \P^{t-1}_i \big)+ \int_{\Omega_{t-1} \times \mathcal X} \KL\big(  L_{0, t}(h_{t-1}, x_t, \cdot\, ), L_{i, t}(h_{t-1}, x_t, \cdot \,)\big)  \dint K_t\P^{t-1}_0 (h_{t-1}, x_t)  \\
			&  = \KL\big( \P^{t-1}_0, \P^{t-1}_i \big) + \int_{ \mathcal X} \KL \!\big(\mathcal N(\phi_0(x_t), 1/ 4), \mathcal N(\phi_i(x_t), 1/ 4) \big) \dint \P^{t-1}_0(x_t) \\
			& = \KL\big( \P^{t-1}_0, \P^{t-1}_i \big) + \E_0 \! \left[ \KL \!\big(\mathcal N(\phi_0(X_t), 1/ 4), \mathcal N(\phi_i(X_t), 1/ 4) \big) \right]
		\end{split}
	\end{equation*}
	where the penultimate equality comes from the fact that the density of the kernel $L_{i, t-1}$ depends only on the last coordinate $x_t$, and is exactly that of a gaussian variable.
	
	We obtain the $\KL$ decomposition by iterating $T$ times,
	\begin{equation*}
		\KL\big( \P^T_0, \P^T_i \big) =  \E_0 \! \left[ \sum_{t =1}^{T} \KL \!\big(\mathcal N(\phi_0(X_t), 1/ 4), \mathcal N(\phi_i(X_t), 1/ 4) \big) \right]  
	\end{equation*}
	
	\paragraph{Continuation of the proof} Let us also define $N_{H_i}(T) = \sum_{t = 1}^{T} \mathds 1_{\{X_t \in H_i\}}$ the number of times the algorithm selects an arm in $H_i$. The hypotheses $\phi_i$ were defined for the three following inequalities to hold. For all $i\geq 1$:
	\begin{equation}\label{eq:rts}
	\overline R_{T, i}
	\geq \frac{\Delta}{2} \left( T - \E_i\big[N_{H_i}(T)\big]\right)
	= \frac{T\Delta}{2} \left( 1 - \frac{\E_i\big[N_{H_i}(T)\big]}{T} \right) \, ,
	\end{equation}
	\begin{equation}\label{eq:reghyp0}
	\overline R_{T, 0}
	\geq \frac{\Delta}{2} \sum_{i= 1}^{K} \E_0\big[N_{H_i}(T)\big] \, ,
	\end{equation}
	and
	\begin{equation}\label{eq:klvalue}
	\begin{split}
	\KL(\P^T_0, \P^T_i) 
	&=  \E_0 \! \left[ \sum_{t =1}^{T} \KL \!\big(\mathcal N(\phi_0(X_t), 1/ 4), \mathcal N(\phi_i(X_t), 1/ 4) \big) \right]  \\
	&= \E_0\! \left[ \sum_{t = 1}^T 2\big(\phi_0(X_t) -  \phi_i(X_t) \big)^2  \right] 
	\leq 2 \,  \E_0\big[N_{H_i}(T)\big] \Delta^2\, .
	\end{split}
	\end{equation}
	The first two inequalities come from the fact that, under $\P_i$, the player incurs an instantaneous regret of less than $\Delta/ 2$ whenever she picks an arm outside the optimal cell $H_i$. For the third inequality, first apply the chain rule to compute the Kullback-Leibler divergence, then the inequality is a consequence of the fact that $\phi_i$ and $\phi_0$ differ only in $H_i$, and their difference is less than $\Delta$.

	We may now proceed with the calculations. By Lemma~\ref{lem:dataproc} applied to the random variable $N_{H_i}(T) / T$:
	\begin{equation}\label{eq:dataprocnhs}
	\frac{\E_i\big[N_{H_i}(T)\big]}{T}  
	\leq \frac{\E_0\big[N_{H_i}(T)\big]}{T}  + \sqrt{\frac{\KL(\P^T_0, \P^T_i) }{2}} \, .
	\end{equation}
	We will now show that 
	\begin{equation}\label{eq:lowerboundpreop}
		\frac{1}{K} \sum_{i= 1}^{K} \overline R_{T, i}  \geq \frac{T \Delta}{2} \left( 1 - \frac{1}{K} - \sqrt{\frac{\Delta \, \overline R_{T, 0}}{K}   }\right) \,.
	\end{equation}
	Indeed by (in order) averaging \eqref{eq:rts} over $i = 1, \dots, K$, using \eqref{eq:dataprocnhs}, the concavity of $\sqrt{\cdot}$ and \eqref{eq:klvalue}
	\begin{equation*}
	\begin{split}
	\frac{1}{K} \sum_{i= 1}^{K} \overline R_{T, i}
	&\geq \frac{T \Delta}{2}  \left( 1 - \frac{1}{K} \sum_{i=1}^{K}\frac{\E_i\big[N_{H_i}(T)\big]}{T} \right) \\
	&\geq \frac{T \Delta}{2}  \left( 1 - \frac{1}{K} \sum_{i=1}^{K}\frac{\E_0\big[N_{H_i}(T)\big]}{T} - \frac{1}{K} \sum_{i=1}^{K} \sqrt{\frac{\KL(\P^T_0, \P^T_i) }{2}} \right)  \\
	& \geq \frac{T \Delta}{2} \left( 1 - \frac{1}{K} - \sqrt{\frac{1}{2K} \sum_{i = 1}^{K} \KL(\P^T_0, \P^T_i) } \right)   \\
	& \geq \frac{T \Delta}{2} \left( 1 - \frac{1}{K} - \sqrt{\frac{\Delta^2}{K}  \sum_{i=1}^{K} \E_0\big[  N_{H_i}(T)\big] }\right) \, .
	\end{split}
	\end{equation*}
	This yields the claimed inequality \eqref{eq:lowerboundpreop} thanks to \eqref{eq:reghyp0}.
	
	Let us assume for now that $K \geq 2$ and $\phi_0 \in \mathcal H(\ell, \gamma)$. Then by the assumption on the algorithm, $\overline R_{T, 0} \leq B$, and therefore
	\begin{equation}\label{eq:lastrt}
	\frac{1}{K} \sum_{i= 1}^{K} \overline R_{T, i}  \geq \frac{T \Delta}{2} \left( \frac{1}{2} - \sqrt{\frac{\Delta B}{K}   }\right) \,.
	\end{equation}
	
	To optimize this bound, we take $\Delta$ as large as possible, while still ensuring that $\sqrt{\Delta B/ K} $ is small enough, e.g., less than $1/4$. Furthermore, we impose that the $\phi_i$'s belong to $\mathcal H (L, \alpha)$, i.e., by Lemma~\ref{lem:globreg}, that $(\Delta / L)^{1/\alpha} \leq 1 / (4K)$.  This leads to the choice
	\begin{equation*}
		\Delta = c\, L^{1/(\alpha +1)} B^{-\alpha / (\alpha + 1)}
		 \quad \text{and} 
		 \quad K 
		 = \floor{ \frac{1}{4} \left( \frac{\Delta}{L}\right)^{- 1/ \alpha}}
		 =  \floor{\frac{c^{-1/\alpha}}{4}(LB)^{1 / (\alpha + 1)}} \, ,
	\end{equation*}
	with $c = 1/128$. 

	\paragraph{Conclusion,  assuming that $K \geq 2$ and $\phi_0\in \mathcal H(\ell, \gamma)$}
	With this choice of parameters, we have by definition of $\Delta$,
	\begin{equation*}
	\Delta B   = c\,  (LB)^{1 / (\alpha + 1)} \, ,
	\end{equation*}
	and by definition of $K$, since $K \geq (c^{-1 / \alpha} / 8) (LB)^{1 / (\alpha + 1)}$,  
	\begin{equation*}
	\frac{\Delta B}{K} \leq 8 c^{1 + 1 / \alpha } 
	\end{equation*}
	hence, using $c^{1 / (2\alpha)} \leq 1$
	\begin{equation*}
	\sqrt{\frac{\Delta B}{K}} \leq 2\sqrt 2 c^{1/2 + 1/(2\alpha)} \leq 2\sqrt 2 \cdot 2^{- 7 / 2}  =  \frac{1}{4} \, .
	\end{equation*}
	With this in hand, we may now go back to inequality \eqref{eq:lastrt} to see that
	\begin{equation*}
	\frac{1}{K} \sum_{i= 1}^{K} \overline R_{T, i}
	\geq  \frac{T \Delta}{2} \left( \frac{1}{2}- \frac{1}{4} \right)
	\geq \frac{T \Delta}{8} = \frac{c}{8} \; T L^{1/(\alpha + 1 )}B^{-\alpha/(\alpha + 1)} \, .
	\end{equation*}
	By the defintion of $K$, it is always true that $(\Delta / L)^{1/\alpha} \leq 1 / (4K)$, and therefore, by Lemma \ref{lem:globreg}, all the $\phi_i$'s automatically belong to $\mathcal H(L, \alpha)$. Therefore, for all $i$, we have $\sup_{f \in \mathcal H (L, \alpha)} \overline R_T \geq \overline{R}_{T, i}$. Hence, recalling that $c = 1 / 128$,
	\begin{equation*}
	\sup_{f \in \mathcal H(L, \alpha)} \Rt 
	\geq
	\frac{1}{K} \sum_{i= 1}^{K} \overline R_{T, i}
	\geq 
	2^{-10} \,  T L^{1 / (\alpha + 1)} B^{- \alpha /(\alpha + 1)} \, .
	\end{equation*}

	\paragraph{Regularity conditions on the mean-payoff functions $\phi_i$}
	We now check that $K \geq 2$, and that $\phi_0 \in \mathcal H(\ell, \gamma)$. Let us first focus on $\phi_0$. By Lemma \ref{lem:globreg}, it is enough to impose that
	$(\Delta / (2\ell) )^{1/\gamma} \leq 1 / 4$, i.e., that 
	\begin{equation*}
		c \, L^{1 / (\alpha + 1)} B^{- \alpha  / (\alpha + 1)}  / (2 \ell) \leq  (1/4)^\gamma
	\end{equation*}
	that is,
	\begin{equation*}
		L^{1 / (\alpha + 1)} B^{- \alpha  / (\alpha + 1)}
		\leq  2 \ell (1/4)^\gamma  / c =  \ell \, 2^{1 -2\gamma} c^{-1}\, ,
	\end{equation*}
	i.e., when
	\begin{equation*}
		L B^{- \alpha} \leq  \ell^{1 + \alpha }\, 2^{ (1  -2 \gamma) (1 + \alpha)}  c^{- (1 + \alpha )} \, 
	\end{equation*}
	hence, replacing $c$ by its value $c = 2^{-7}$, the next condition is sufficient to ensure the regularity of the hypothesis:
	\begin{equation*}
		L 
		\leq \ell^{1 + \alpha} \,  B^\alpha \,c^{-(1 + \alpha)} \, 2^{(1+\alpha)(1 - 2 \gamma)} 
		= \ell^{1 + \alpha} \,  B^\alpha \, 2^{(1+\alpha)(8 - 2 \gamma)}  \, ,
	\end{equation*}
	which is one of the two conditions in the statement of the theorem.
	For the bound to be valid, we must also make sure that $K \geq 2$:
	\begin{equation*}
		\floor{\left( \frac{c^{-1 / \alpha}}{4} (LB)^{1/(\alpha + 1)}\right)} \geq 2 \, .
	\end{equation*}
	This condition is weaker than
	\begin{equation*}
		\frac{c^{-1 / \alpha}}{4} (LB)^{1/(\alpha + 1)} \geq 3
	\end{equation*}
	which is equivalent to 
	\begin{equation*}
		L \geq c^{(\alpha + 1)/\alpha} \, 12^{\alpha + 1} B^{-1}  = 2^{-7} \cdot  12 \cdot 2^{-6 / \alpha} 12^\alpha B^{-1}\, .
	\end{equation*}
	To ensure this, we require the stronger (but more readable) condition that $L\geq 2^{-3}12^\alpha B^{-1}$. 
\end{proof}

\begin{proof}[Proof of Lemma~\ref{lem:globreg}]
	A good look at Figure~\ref{fig:deltaplot} should convince the reader of the statement. We wish to make sure that the functions $\phi_i$'s satisfy \eqref{eq:holdercondition}, a Hölder condition around their maximum (and only around this maximum). Given the definition of the functions $\phi_i$, we simply have to check that there is no discontinuity at the boundary of the cell $H_i$. We write out the details for $i > 0$ to remove any doubt; the same analysis can be carried to check that $\phi_0 \in \mathcal H(\ell, \gamma)$. 
	
	\begin{figure}[h]
		\centering
		\begin{subfigure}{.5\textwidth}
			\centering
			\includegraphics[width=\linewidth]{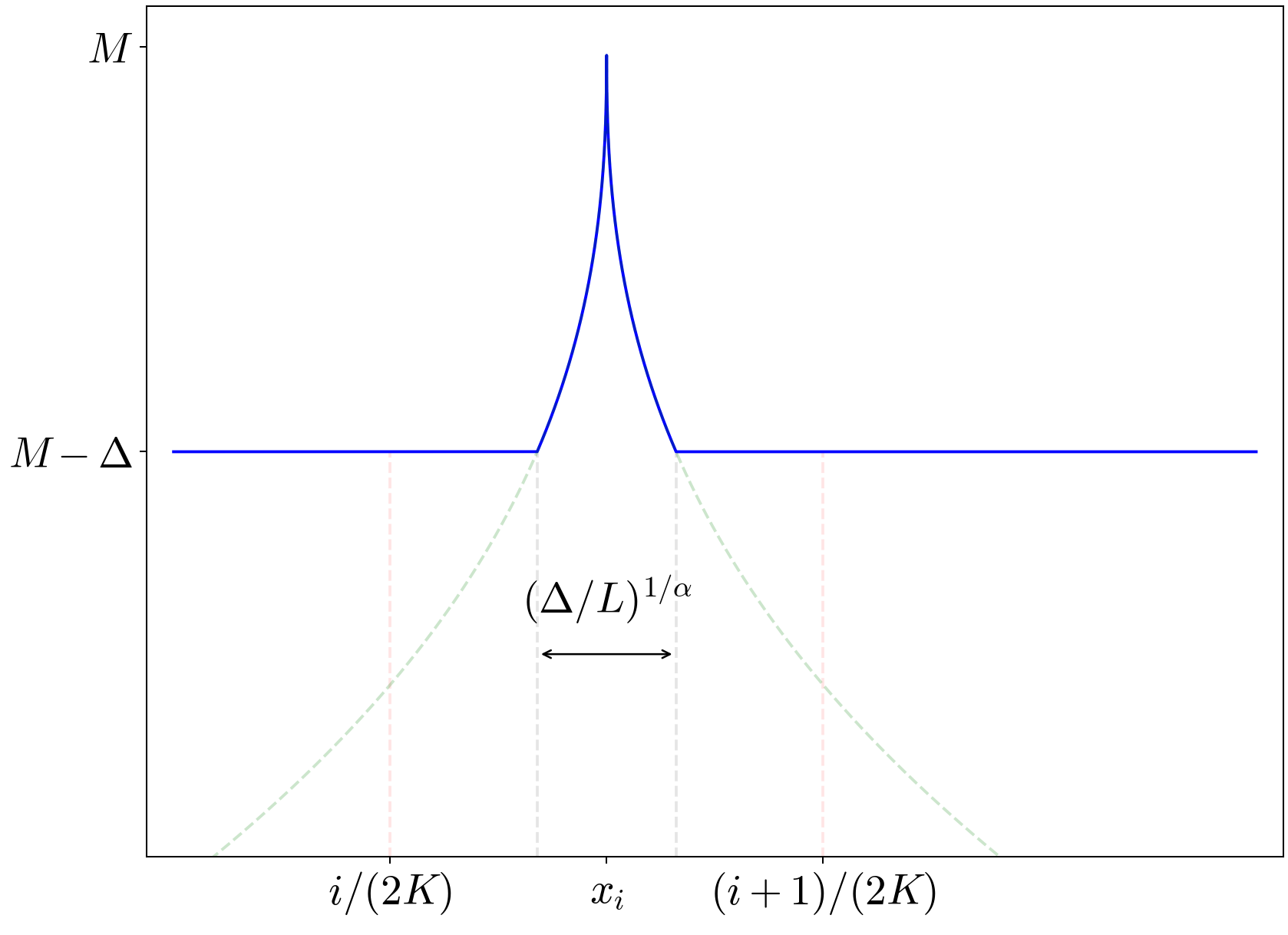}
			\caption{$(\Delta / L)^{1/\alpha} \leq 1 / (4K)$ hence $\phi_i \in \mathcal H(L, \alpha)$}
			\label{fig:sub1}
		\end{subfigure}%
		\begin{subfigure}{.5\textwidth}
			\centering
			\includegraphics[width=\linewidth]{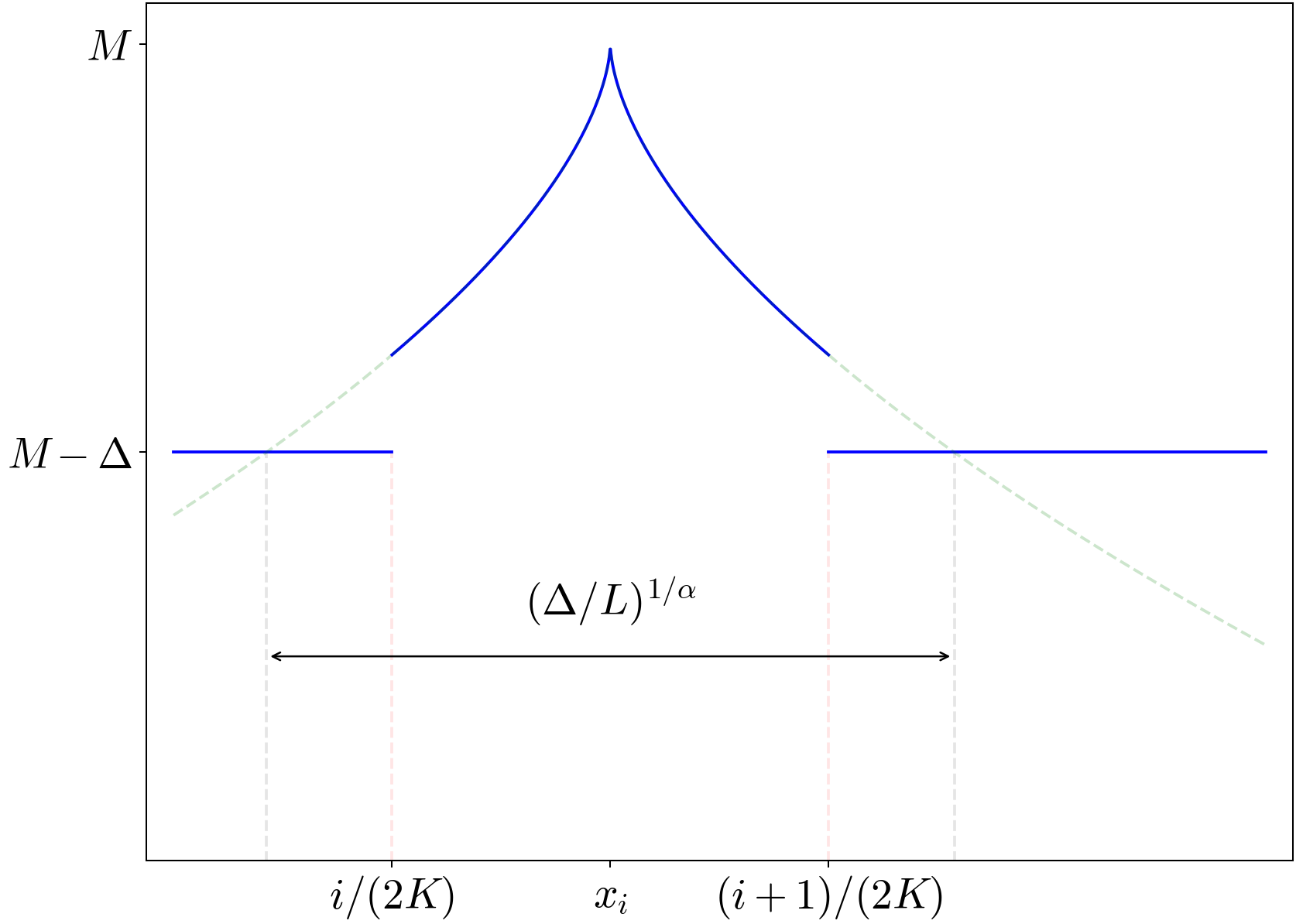}
			\caption{ $(\Delta / L)^{1/\alpha} > 1 / (4K)$ hence $\phi_i \notin \mathcal H(L, \alpha)$}
			\label{fig:sub2}
		\end{subfigure}
		\caption{$\phi_i$ is in $\mathcal H(L, \alpha) $ if it is everywhere above the green dotted curve $x \mapsto M - L \abs{x - x_i}^{\alpha}$, that is, if the cell $H_i$ has enough room to contain the whole peak of size $\Delta$}
		\label{fig:deltaplot}
	\end{figure}

	For $i > 0$, the function $\phi_i$ reaches its maximum at $x_i = (i - 1/2)/ 2K$, and the value of the maximum is $M$. Then for $x \in H_i$, by definition of $\phi_i$:
	\begin{equation*}
		\phi_i(x) = \max \big(M - \Delta, M - L \abs{x_i - x}^\alpha\big) \geq  M - L \abs{x - x_i}^\alpha
	\end{equation*}
	thus
	\begin{equation*}
		\phi_i(x_i) - \phi_i(x) = M - \phi_i(x) \leq L \abs{x_i - x}^\alpha \, ,
	\end{equation*}
	Now consider $x \notin H_i$. Assume, as in the statement of the lemma, that $1 / (4K) \geq (\Delta / L)^{1 / \alpha}$. If $x$ is outside of $H_i$, then since $H_i$ is of half-width $1/4K$, 
	\begin{equation}\label{eq:last}
		\abs{x_i-x} \geq \frac{1}{4K} \geq \left(\frac{\Delta}{L} \right)^{1/\alpha}
	\end{equation}
	and, by definition of $\phi_i$, for all $x$ (even for $x \in H_0$), 
	$
		\phi_i(x) \geq M - \Delta\, .
	$
	Therefore, by \eqref{eq:last}, 
	\begin{equation*}
		\phi_i(x_i) - \phi_i(x) \leq \Delta \leq L \abs{x_i - x}^{\alpha}\, .
	\end{equation*}
	For all values of $x$, the Hölder condition is satisfied and $\phi_i \in \mathcal H (L, \alpha)$.
	
	For $\phi_0$, the same calculations show that there is no jump at the boundary of $[1/2, 1]$, of half-width $1/4$, when the peak is of height $\Delta / 2$ and regularity $(\ell,\gamma)$ if $\big((\Delta / 2) / \ell)\big)^{1/ \gamma} \leq 1/ 4$.
\end{proof}

%% file: final_bandit_adaptation.bbl
\begin{thebibliography}{24}
\providecommand{\natexlab}[1]{#1}
\providecommand{\url}[1]{\texttt{#1}}
\expandafter\ifx\csname urlstyle\endcsname\relax
  \providecommand{\doi}[1]{doi: #1}\else
  \providecommand{\doi}{doi: \begingroup \urlstyle{rm}\Url}\fi

\bibitem[Agrawal(1995)]{agrawal1995ContinuumArmed}
Rajeev Agrawal.
\newblock The {{Continuum}}-{{Armed Bandit Problem}}.
\newblock \emph{SIAM Journal on Control and Optimization}, 33\penalty0
  (6):\penalty0 1926--1951, 1995.
\newblock ISSN 0363-0129, 1095-7138.
\newblock \doi{10.1137/S0363012992237273}.
\newblock URL \url{http://epubs.siam.org/doi/10.1137/S0363012992237273}.

\bibitem[Audibert and Bubeck()]{audibert2009Minimax}
Jean-Yves Audibert and S{\'e}bastien Bubeck.
\newblock Minimax policies for adversarial and stochastic bandits.
\newblock In \emph{{{COLT}}}, pages 217--226.

\bibitem[Auer et~al.()Auer, Ortner, and Szepesv{\'a}ri]{auer2007improved}
Peter Auer, Ronald Ortner, and Csaba Szepesv{\'a}ri.
\newblock Improved rates for the stochastic continuum-armed bandit problem.
\newblock In \emph{International {{Conference}} on {{Computational Learning
  Theory}}}, pages 454--468. {Springer}.

\bibitem[Birg{\'e} and Massart(1995)]{birge1995Estimation}
Lucien Birg{\'e} and Pascal Massart.
\newblock Estimation of {{Integral Functionals}} of a {{Density}}.
\newblock \emph{The Annals of Statistics}, 23\penalty0 (1):\penalty0 11--29,
  1995.
\newblock URL \url{https://doi.org/10.1214/aos/1176324452}.

\bibitem[Bubeck et~al.()Bubeck, Stoltz, and Yu]{bubeck2011Lipschitza}
S{\'e}bastien Bubeck, Gilles Stoltz, and Jia~Yuan Yu.
\newblock Lipschitz bandits without the {{Lipschitz}} constant.
\newblock In \emph{International {{Conference}} on {{Algorithmic Learning
  Theory}}}, pages 144--158. {Springer}.

\bibitem[Bubeck et~al.(2011{\natexlab{a}})Bubeck, Munos, and
  Stoltz]{bubeck2011Pure}
S{\'e}bastien Bubeck, R{\'e}mi Munos, and Gilles Stoltz.
\newblock Pure exploration in finitely-armed and continuous-armed bandits.
\newblock \emph{Theoretical Computer Science}, 412\penalty0 (19):\penalty0
  1832--1852, 2011{\natexlab{a}}.
\newblock ISSN 03043975.
\newblock \doi{10.1016/j.tcs.2010.12.059}.
\newblock URL
  \url{https://linkinghub.elsevier.com/retrieve/pii/S030439751000767X}.

\bibitem[Bubeck et~al.(2011{\natexlab{b}})Bubeck, Munos, Stoltz, and
  Szepesv{\'a}ri]{bubeck2011Xarmed}
S{\'e}bastien Bubeck, R{\'e}mi Munos, Gilles Stoltz, and Csaba Szepesv{\'a}ri.
\newblock X-armed bandits.
\newblock \emph{Journal of Machine Learning Research}, 12:\penalty0 1655--1695,
  2011{\natexlab{b}}.

\bibitem[Bull(2015)]{bull2015Adaptivetreed}
Adam~D. Bull.
\newblock Adaptive-treed bandits.
\newblock \emph{Bernoulli}, 21\penalty0 (4):\penalty0 2289--2307, 2015.
\newblock \doi{10.3150/14-BEJ644}.
\newblock URL \url{https://doi.org/10.3150/14-BEJ644}.

\bibitem[Cai(2012)]{cai2012Minimax}
T.~Tony Cai.
\newblock Minimax and {{Adaptive Inference}} in {{Nonparametric Function
  Estimation}}.
\newblock \emph{Statistical Science}, 27\penalty0 (1):\penalty0 31--50, 2012.
\newblock \doi{10.1214/11-STS355}.
\newblock URL \url{https://doi.org/10.1214/11-STS355}.

\bibitem[Cai and Low()]{cai2005Adaptive}
T.~Tony Cai and Mark~G. Low.
\newblock On adaptive estimation of linear functionals.
\newblock 33\penalty0 (5):\penalty0 2311--2343.
\newblock \doi{10.1214/009053605000000633}.
\newblock URL \url{https://doi.org/10.1214/009053605000000633}.

\bibitem[Coquelin and Munos(2007)]{coquelin2007bandit}
Pierre-Arnaud Coquelin and R{\'e}mi Munos.
\newblock Bandit algorithms for tree search.
\newblock In \emph{Uncertainty in Artificial Intelligence}, 2007.

\bibitem[Garivier et~al.()Garivier, Hadiji, Menard, and
  Stoltz]{garivier2018KLUCBswitch}
Aur{\'e}lien Garivier, H{\'e}di Hadiji, Pierre Menard, and Gilles Stoltz.
\newblock {{KL}}-{{UCB}}-switch: Optimal regret bounds for stochastic bandits
  from both a distribution-dependent and a distribution-free viewpoints.
\newblock URL \url{http://arxiv.org/abs/1805.05071}.

\bibitem[Garivier et~al.(2018)Garivier, M{\'e}nard, and
  Stoltz]{garivier2018Explore}
Aur{\'e}lien Garivier, Pierre M{\'e}nard, and Gilles Stoltz.
\newblock Explore {{First}}, {{Exploit Next}}: {{The True Shape}} of {{Regret}}
  in {{Bandit Problems}}.
\newblock \emph{Mathematics of Operations Research}, 2018.
\newblock ISSN 0364-765X.
\newblock \doi{10.1287/moor.2017.0928}.
\newblock URL \url{https://pubsonline.informs.org/doi/10.1287/moor.2017.0928}.

\bibitem[Grill et~al.()Grill, Valko, and Munos]{grill2015Blackbox}
Jean-Bastien Grill, Michal Valko, and R{\'e}mi Munos.
\newblock Black-box optimization of noisy functions with unknown smoothness.
\newblock In \emph{Advances in {{Neural Information Processing Systems}}},
  pages 667--675.

\bibitem[Kallenberg(2006)]{kallenberg2006foundations}
Olav Kallenberg.
\newblock \emph{Foundations of modern probability}.
\newblock Springer Science \& Business Media, 2006.

\bibitem[Kleinberg()]{kleinberg2004Nearly}
Robert Kleinberg.
\newblock Nearly {{Tight Bounds}} for the {{Continuum}}-armed {{Bandit
  Problem}}.
\newblock In \emph{Proceedings of the 17th {{International Conference}} on
  {{Neural Information Processing Systems}}}, pages 697--704. {MIT Press}.

\bibitem[Kleinberg et~al.()Kleinberg, Slivkins, and
  Upfal]{kleinberg2013Bandits}
Robert Kleinberg, Aleksandrs Slivkins, and Eli Upfal.
\newblock Bandits and {{Experts}} in {{Metric Spaces}}.
\newblock abs/1312.1277.
\newblock URL \url{http://arxiv.org/abs/1312.1277}.

\bibitem[Krishnamurthy et~al.()Krishnamurthy, Langford, Slivkins, and
  Zhang]{krishnamurthy2019Contextual}
Akshay Krishnamurthy, John Langford, Aleksandrs Slivkins, and Chicheng Zhang.
\newblock Contextual {{Bandits}} with {{Continuous Actions}}: {{Smoothing}},
  {{Zooming}}, and {{Adapting}}.
\newblock URL \url{http://arxiv.org/abs/1902.01520}.

\bibitem[Lattimore and Szepesv\'{a}ri(2019)]{lattimore2019book}
Tor Lattimore and Csaba Szepesv\'{a}ri.
\newblock \emph{Bandit Algorithms}, volume Bubeck, S{\'e}bastien and Munos,
  R{\'e}mi and Stoltz, Gilles and Szepesv{\'a}ri, Csaba.
\newblock Cambridge University Press (preprint), 2019.

\bibitem[Lepskii(1991)]{lepskii1991Problem}
O.~Lepskii.
\newblock On a {{Problem}} of {{Adaptive Estimation}} in {{Gaussian White
  Noise}}.
\newblock \emph{Theory of Probability \& Its Applications}, 35\penalty0
  (3):\penalty0 454--466, 1991.
\newblock ISSN 0040-585X.
\newblock \doi{10.1137/1135065}.
\newblock URL \url{https://epubs.siam.org/doi/10.1137/1135065}.

\bibitem[Locatelli and Carpentier()]{locatelli2018adaptivity}
Andrea Locatelli and Alexandra Carpentier.
\newblock Adaptivity to {{Smoothness}} in {{X}}-armed bandits.
\newblock In \emph{Conference on {{Learning Theory}}}, pages 1463--1492.

\bibitem[Massart()]{massart2007Concentration}
Pascal Massart.
\newblock \emph{Concentration {{Inequalities}} and {{Model Selection}}}.
\newblock Ecole d'{{Et{\'e}}} de {{Probabilit{\'e}s}} de {{Saint}}-{{Flour
  XXXIII}} - 2003. {Springer Berlin Heidelberg}.
\newblock ISBN 978-3-540-48503-2.
\newblock URL \url{https://doi.org/10.1007/978-3-540-48503-2}.

\bibitem[Shang et~al.()Shang, Kaufmann, and Valko]{shang2019General}
Xuedong Shang, Emilie Kaufmann, and Michal Valko.
\newblock General parallel optimization without a metric.
\newblock In \emph{Algorithmic {{Learning Theory}}}, volume~98.
\newblock URL \url{https://hal.inria.fr/hal-02047225}.

\bibitem[Valko et~al.()Valko, Carpentier, and Munos]{valko2013Stochastic}
Michal Valko, Alexandra Carpentier, and R{\'e}mi Munos.
\newblock Stochastic {{Simultaneous Optimistic Optimization}}.
\newblock In \emph{International {{Conference}} on {{Machine Learning}}}.

\end{thebibliography}
